\documentclass{article}

\PassOptionsToPackage{round}{natbib}

\usepackage[final]{neurips_2025}

\usepackage[utf8]{inputenc} 
\usepackage[T1]{fontenc}    
\usepackage{hyperref}       
\usepackage{url}            
\usepackage{booktabs}       
\usepackage{amsfonts}       
\usepackage{nicefrac}       
\usepackage{microtype}      
\usepackage{xcolor}         
\usepackage{subcaption}

\usepackage{mathtools}
\usepackage{amsmath,amsfonts,amssymb,dsfont,bm,amsthm}
\usepackage{csquotes}
\usepackage{float}
\usepackage{subcaption}
\usepackage{graphicx}
\usepackage{txfonts}
\usepackage{dsfont}
\usepackage{hyperref}
\usepackage{thm-restate}
\usepackage{soul}
\usepackage{fix-cm}
\usepackage{comment}
\usepackage{soul}

\def\*#1{\mathcal{#1}}

\newtheorem{theorem}{Theorem}
\newtheorem{lemma}{Lemma}
\newtheorem{corollary}{Corollary}
\newtheorem{proposition}{Proposition}

\newtheorem{definition}{Definition}
\newtheorem{example}{Example}

\newtheorem{notation}{Notation}

\usepackage{slashed}
\newcommand{\ind}{{\perp\!\!\!\perp}}
\newcommand{\leftcomingarrow}{{\leftarrow \! \! \! \! \bullet}}
\newcommand{\rightcomingarrow}{\bullet \! \! \! \! \rightarrow}
\newcommand{\notind}{\slashed{\ind}}
\newcommand{\V}{\mathcal{V}}
\newcommand{\E}{\mathcal{E}}
\newcommand{\X}{\mathcal{X}}
\newcommand{\Y}{\mathcal{Y}}
\newcommand{\Z}{\mathcal{Z}}

\newcommand{\Pa}{\text{Pa}}
\newcommand{\Anc}{\text{Anc}}
\newcommand{\Desc}{\text{Desc}}
\newcommand{\G}{{\mathcal{G}}}
\newcommand{\Gm}{{\G^m}}
\newcommand{\Gc}{{\G^C}}
\newcommand{\C}{\mathcal{C}}
\newcommand{\U}{\mathcal{U}}
\newcommand{\Do}{\text{do}}

\newcommand{\Gcu}{{{\G^{m}_{\text{u}}}}}

\newcommand{\Gt}{{\G^m_{\text{can}}}}
\newcommand{\Root}{\mathrm{Root}}
\newcommand{\A}{\mathcal{A}}
\newcommand{\B}{\mathcal{B}}
\newcommand{\R}{\mathcal{R}}
\newcommand{\M}{\mathcal{M}}
\newcommand{\F}{\mathcal{F}}
\newcommand{\W}{\mathcal{W}}
\newcommand{\NF}{\mathcal{NF}}

\usepackage{tikz}
\usetikzlibrary{babel, arrows, decorations, decorations.pathmorphing, decorations.pathreplacing, decorations.shapes, calc}
\usetikzlibrary{arrows.meta, positioning}
\usetikzlibrary{fit}
\usetikzlibrary{shapes.geometric}

\newcommand{\leftselfloop}{
    \!\!\!
    \begin{tikzpicture}[baseline=(X.base)]
        \node (X) at (0,0) {$\phantom{X}$};
        \path (X) edge [->, loop left, looseness=4, in=155, out=205] node {} (X);
    \end{tikzpicture} 
    \!\!\!\!\!\!\!
}

\tikzset{tochoose/.style={->,dotted, thick, >=stealth}}

\usepackage[vlined]{algorithm2e}
\RestyleAlgo{ruled}
\SetKwComment{Comment}{/* }{ */}

\title{Relaxing partition admissibility in Cluster-DAGs: \newline a causal calculus with arbitrary variable clustering}

\author{%
Clément Yvernes\\
  Univ Grenoble Alpes, CNRS,\\ Grenoble INP, LIG\\
\And
Emilie Devijver\\
  Univ Grenoble Alpes, CNRS,\\ Grenoble INP, LIG\\
\And
Adèle H. Ribeiro\\
  University of Münster,\\ Institute of Medical Informatics\\
\And
Marianne Clausel\\
  Université de Lorraine, CNRS,\\ CRAN, 54000 Nancy\\
\And
Eric Gaussier\\
  Univ Grenoble Alpes, CNRS,\\ Grenoble INP, LIG\\
}

\begin{document}

\maketitle

\begin{abstract}
Cluster DAGs (C-DAGs) provide an abstraction of causal graphs in which nodes represent clusters of variables, and edges encode both cluster-level causal relationships and dependencies arisen from unobserved confounding. C-DAGs define an equivalence class of acyclic causal graphs that agree on cluster-level relationships, enabling causal reasoning at a higher level of abstraction. However, when the chosen clustering induces cycles in the resulting C-DAG, the partition is deemed inadmissible under conventional C-DAG semantics. In this work, we extend the C-DAG framework to support arbitrary variable clusterings by relaxing the partition admissibility constraint, thereby allowing cyclic C-DAG representations. We extend the notions of d-separation and causal calculus to this setting, significantly broadening the scope of causal reasoning across clusters and enabling the application of C-DAGs in previously intractable scenarios. Our calculus is both sound and atomically complete with respect to the do-calculus: all valid interventional queries at the cluster level can be derived using our rules, each corresponding to a primitive do-calculus step. 
\end{abstract}

\label{sec:intro}

Knowing the effect of a treatment $X$ on an outcome $Y$, encoded by an intervention $do(X)$ in the interventional distribution $P(Y|do(X))$, is crucial in many applications. However, performing interventions is often impractical due to ethical concerns, potential harm or prohibitive costs. In such cases, one can instead aim to identify do-free formulas that estimate the effects of interventions using only   observational (non-experimental) data and  a causal graph \citep{pearl_causality_2009}. Solving the identifiability problem typically involves establishing graphical criteria under which the total effect is identifiable, and providing a do-free formula for estimating it from observational data. However, specifying a causal diagram requires prior knowledge of the causal relationships between all observed variables, a requirement that is often unmet in real-world applications. This challenge is particularly acute in complex, high-dimensional settings, limiting the practical applicability of causal inference methods.

One way to circumvent this difficulty is to rely on abstract representations which group several variables, a mapping usually referred to as causal representation learning \citep{scholkopf_CRL}, which are connected through causal relationships and dependencies arisen from unobserved confounding.
Several studies have been devoted to causal discovery of and causal inference in specific abstract representations, both for static and dynamic (time series) variables, as \citet{assaad2022,ferreira2024,anand_causal_2023,wahl2023}. Recent studies have also tackled the related problem of defining mappings from clusters to variables while preserving specific causal properties \citep{chalupka15visual, chalupka16multi, rubensteinWBMJG17, halpern19abstraction}, and have provided a thorough theoretical analysis of the relationship between micro- and macro-level causal models with a view on causal discovery assumptions \citep{wahl2024}. 

The starting point of our study is the framework recently proposed in \citet{anand_causal_2023}, which relaxes the strict requirement of a fully specified causal diagram and provides a foundation for valid inference over clusters of variables. However, it focuses on abstract graphs, called Cluster-DAGs, which do not contain cycles between clusters, a restriction known as \textit{partition admissibility}. 
We extend this framework in this paper by removing this restriction and consider arbitrary clusterings of variables, potentially resulting in abstract graphs with self-loops and cycles between clusters.
Several real‑world scenarios illustrate this concept as in macroeconomics where sector‑level relationships are known (as consumption $\rightarrow$ investment) but not the firm‑to‑firm or household‑to‑household causal links, or in neuroscience where functional MRI region interactions can be established but not necessarily the causal links between neurons.

To tackle identification in causal abstractions, recent work has introduced separation criteria that generalize d-separation to specific abstraction types \citep{jaber_causal_2024, perkovic_complete_2018, Ferreira_Assaad_2025}. While effective, this has led to a proliferation of abstraction-specific rules. Yet, all these methods operate over the same underlying object: the class of graphs compatible with the abstraction. An alternative line of work could seek to construct a transformed graph on which standard d-separation can be directly applied. This is feasible, for example, when the union of all compatible graphs is itself compatible — but such cases are rare, especially when the abstraction introduces cycles.
Our approach adopts a hybrid strategy. Rather than relying solely on path-based d-separation, we define a structure-based criterion that captures all necessary information for assessing separation. These structures avoid pitfalls such as reattaching colliders into conditioned paths and are simple to construct — typically by tracing backward along directed edges from root nodes. Crucially, our criterion remains tightly aligned with standard d-separation: every d-connecting path {can define} such a structure, and every connecting structure contains a d-connecting path. To support identification under abstraction, we introduce a two-step method grounded in a tractable search space. First, we derive an extended graph from the abstraction that, while not necessarily compatible, conservatively includes all potentially connecting structures. This enables efficient exploration using standard graph-traversal techniques. Second, a lightweight compatibility test is applied to filter out invalid structures, \textit{i.e.}, those which do not correspond to any graph in the compatible class.

Specifically, we make the following contributions:
\begin{enumerate}
    \item  We extend the framework of \citet{anand_causal_2023} by removing the assumption of partition admissibility, thereby broadening its applicability to a wider range of causal abstractions.
    \item We reformulate the d-separation criterion in an ADMG using a structure-based separation criterion that remains faithful to classical d-separation.
    \item We introduce a calculus which is sound and atomically complete.
    \item We further show that any cluster can be reduced to a cluster of limited size, leading to efficient calculus rules. 
\end{enumerate}

The remainder of the paper is structured as follows:  Section \ref{sec:notions} introduces the main notions while Section \ref{sec:calculus} presents our main result regarding sound and atomically complete calculus; Section \ref{sec:3nodes} presents an efficient way to look at causal abstractions based on clusters; lastly, Section~\ref{sec:discussion} discusses some extensions of our work while Section~\ref{sec:conclusion} concludes the paper. All proofs are provided in the Technical Appendices.

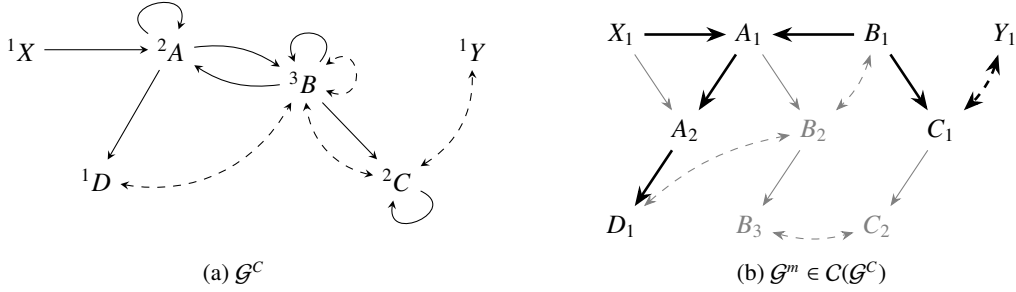
\begin{figure}[t]
    \centering
    \begin{subfigure}{0.45\textwidth}
        \centering
        \begin{tikzpicture}[ ->, >=stealth, scale=2]
        
            \node (X) at (-.5, 0.866) {${}^1X$};
            \node (A) at (0.5, 0.866) {${}^2A$};
            \node (B) at (1.3749, 0.6497) {${}^3B$};
            \node (C) at (2, 0) {${}^2C$};
            \node (D) at (0, 0) {${}^1D$};
            \node (Y) at (2.5, 0.866) {${}^1Y$};

            \draw (X) edge (A);
            \path (C) edge [<->, dashed, bend right = 30] node {} (Y);
            \path (D) edge [<->, dashed, bend right = 35] node {} (B);
            \path (C) edge [<->, dashed, bend left = 35] node {} (B);
            \path (A) edge[bend left=20] (B);
            \path (B) edge[bend left=20] (A);
            \draw (B) -- (C);
            \draw (A) -- (D);
            \path (C) edge [->, loop right, looseness=4, in=260, out=335] node {} (C);
            \path (A) edge [->, loop right, looseness=4, in=70, out=145] node {} (A);
            \path (B) edge [->, loop right, looseness=4.5, in=112, out=50] node {} (B);
            \path (B) edge [<->,dashed, loop right, looseness=4.5, in=37, out=-20] node {} (B);
        \end{tikzpicture}
        \caption{$\Gc$}
        \label{fig:example_cluster_DAG:a}
    \end{subfigure}
    \hfill
    \begin{subfigure}{0.45\textwidth}
    \centering
    \begin{tikzpicture}[->, >=Stealth, scale = .85]
    
    \node[text=black, font=\bfseries] (X) at (0,0) {$X_1$};
    \node[text=black, font=\bfseries] (A) at (2,0) {$A_1$};
    \node[text=black, font=\bfseries] (B) at (4,0) {$B_1$};
    \node[text=black, font=\bfseries] (Y) at (6,0) {$Y_1$};
    \node[text=black, font=\bfseries] (C) at (1,-1.5) {$A_2$};
    \node[text=gray] (D) at (3,-1.5) {$B_2$};
    \node[text=black, font=\bfseries] (E) at (5,-1.5) {$C_1$};
    \node[text=black, font=\bfseries] (F) at (0,-3) {$D_1$};
    \node[text=gray] (G) at (2,-3) {$B_3$};
    \node[text=gray] (H) at (4,-3) {$C_2$};
    
    \draw[->, line width=1pt] (X) -- (A);
    \draw[->, line width=1pt] (B) -- (A);
    \draw[->, line width=1pt] (A) -- (C);
    \draw[->, draw=gray] (X) -- (C);
    \draw[->, draw=gray] (A) -- (D);
    \draw[->, line width=1pt] (B) -- (E);
    \draw[->, line width=1pt] (C) -- (F);
    \draw[->, draw=gray] (D) -- (G);
    \draw[->, draw=gray] (E) -- (H);
    
    \draw[<->, dashed, draw=gray] (B) to[bend left=15] (D);
    \draw[<->, dashed, line width=1pt] (Y) to[bend left=15] (E);
    \draw[<->, dashed, draw=gray] (F) to[bend left=15] (D);
    \draw[<->, dashed, draw=gray] (G) to[bend right=15] (H);
    
    \end{tikzpicture}
    \caption{$\Gm \in \C(\Gc)$}
    \label{fig:example_cluster_DAG:b}
    \end{subfigure}
    \caption{Left: a C-DAG $\Gc = (\V^C, \E^C)$, Right: a graph $\Gm=(\V^m, \E^m)$ that is compatible with $\Gc$. For example, ${}^2A \in \V^C$ corresponds to $\{A_1,A_2\} \subseteq \V^m$. In $\Gm$, a \textit{structure of interest} (Definition \ref{def:structure_of_interest}) is highlighted in \textbf{bold black}, with all other nodes and edges shown in gray.}
    \label{fig:example_cluster_DAG}
\end{figure}

\section{Preliminaries}
\label{sec:notions}

We follow the notations of \citet{pearl_causality_2009}.
A single variable is denoted by an uppercase letter $X$ and its realized value by a small letter $x$. 
A calligraphic uppercase letter $\X$ denotes a set.  

\paragraph{Graphs.} We denote by $\Anc(\X, \G)$ and $\Desc(\X, \G)$ the sets of ancestors and descendants of $\X$ in the graph ${\G}$, respectively. By convention, each node is regarded as its own ancestor and descendant. We denote by $\Root(\G)$ the set of roots of $\G$, i.e., the vertices that have no child in $\G$.
%
A vertex $V$ is said to be \emph{active} on a path relative to a subset of variables $\Z$ if 1) $V$ is a collider and $V$ or any of its descendants are in $\Z$ or 2) $V$ is a non-collider and is not in $\Z$. A path $\pi$  is said to be \emph{active} given (or conditioned on) $\Z$ if every vertex on $\pi$ is active relative to $\Z$. Otherwise, $\pi$   is said to be \emph{inactive} given $\Z$. 
%
Given a graph $\mathcal{G}$, the sets $\X$ and $\Y$ are said to be d-separated by $\Z$ if every path between $\X$ and $\Y$ is inactive given $\Z$. We denote this by $\X \ind_{\G} \Y \mid \Z$. Otherwise, $\X$ and $\Y$ are d-connected given $\Z$, which we denote by $\X \notind_{\G} \Y \mid \Z$.
%
The mutilated graph 
 ${\mathcal{G}}_{\overline{\X}\underline{\Z}}$ is the result of
removing from a graph ${\mathcal{G}}$ edges 
with an arrowhead into $\X$ (e.g., $A \rightarrow \X$, $A \leftrightarrow \X$),  
and edges with a tail from $\Z$ (e.g., $A \leftarrow \Z$). Let $\pi$ be a path in a graph $\G$ and let A and B be two nodes of $\pi$. We denote by $\pi_{[A,B]}$, the subpath of $\pi$ between $A$ and $B$. For two graphs $\G_1 = (\V_1, \E_1)$ and $\G_2 = (\V_2, \E_2)$, the union is $\G_1 \cup \G_2 \coloneqq (\V_1 \cup \V_2, \E_1 \cup \E_2)$. We denote by $\mathcal{X} \cap \G$ the set of nodes in $\G$ that belong to $\mathcal{X}$, i.e., the intersection between $\mathcal{X}$ and the vertex set of $\G$. We denote by $\G \setminus \mathcal{X}$ the subgraph of $\G$ obtained by removing all vertices in $\mathcal{X}$ together with any edges incident to $\mathcal{X}$.

\paragraph{Structural Causal Models.} Formally, a {Structural Causal Model} (SCM) $\mathcal{M}$ is a 4-tuple $\langle \U, \V, \mathcal{F}, P(\U)\rangle$, where $\U$ is a set of exogenous (latent) mutually independent variables and $\V$ is a set of endogenous (measured) variables\footnote{The induced distribution $P$ is Markovian with respect to the graph associated to the SCM (Proposition 6.31 in \cite{Peters-book}).}. $\mathcal{F}$ is a collection of functions $\{f_i\}_{i=1}^{|\V|}$ such that each endogenous variable $V_i\in\V$ is a function $f_i\in\mathcal{F}$ of $\U_i\cup \Pa(V_i)$, 
where $\U_i\subseteq\U$ and $\Pa(V_i)\subseteq\V\setminus V_i$. The uncertainty is encoded through a probability distribution over the exogenous variables, $P(\U)$.
Each SCM $\mathcal{M}$ induces a directed acyclic graph (DAG) with bidirected edges -- or an acyclic directed mixed graph (ADMG) -- $G(\V, \*E = (\E_{D},\E_{B}))$, known as a \emph{causal diagram}, that encodes the structural relations among $\V\cup\U$, where every $V_i \in \V$ is a vertex. 
We potentially distinguish edges $\mathcal{E}$ into directed edges $\E_D$, which connect each variable $V_i \in \mathcal{V}$ to its parents $V_j \in \Pa(V_i)$ as $(V_j \rightarrow V_i)$, and bidirected edges, which appear as dashed edges $(V_j \dashleftrightarrow V_i)$ between variables $V_i, V_j \in \mathcal{V}$ that share a common exogenous parent, i.e., such that $\mathcal{U}_i \cap \mathcal{U}_j \neq \emptyset$.
%
Performing an intervention $\X\!\!=\!\! x$ is represented through the do-operator, \textit{do}($\X\!=\!x$), which represents the operation of fixing a set $\X$ to a constant $x$, and 
induces a submodel $\mathcal{M}_\X$, which is $\mathcal{M}$ with $f_X$ replaced to $x$ for every $X \in \X$. The post-interventional distribution induced by $\mathcal{M}_\X$ is denoted by $P(\V \setminus \X |do(\X))$. 

\paragraph{Cluster-DAGs}

In this study, we further develop the Cluster-DAG framework introduced in \cite{anand_causal_2023}. The individual variables, called micro-variables and denoted by $\V^m$, are grouped into clusters forming a partition $\V^C$, where each cluster contains one or more micro-variables.
\begin{definition}[Cluster-DAG]
\label{def:Cluster_DAG}
Let \(\G^m = (\V^m, \E^m)\) be an ADMG and let \(\V^C\) be a partition of \(\V^m\). We construct the mixed graph \(\Gc = (\V^C, \E^C)\), possibly with self-loops and cycles, by defining \(\E^C\) as follows. For all clusters \(V, W \in \V^C\):
\begin{itemize}
    \item $V \rightarrow W$ is in $\E^C$ \emph{if and only if} there exists $V_v$ in $V$ and $W_w$ in $W$ such that $V_v \rightarrow W_w$ in $\G^m$.
    \item $V \dashleftrightarrow W$ is in $\E^C$ \emph{if and only if} there exists $V_v$ in $V$ and $W_w$ in $W$ such that $V_v \dashleftrightarrow W_w$ in $\G^m$.
\end{itemize}
We say that $\G^m$ and $\G^C$ are compatible. 
\end{definition}

A key novelty of our approach is the allowance of cycles at the cluster level, in contrast to \cite{anand_causal_2023}, which restricts the cluster graph to be acyclic. We nonetheless retain the term “Cluster-DAG” (C-DAG) to emphasize that the underlying graph on micro-variables remain acyclic.
We denote a node in $\Gc$ by $V^C$, and its corresponding set of micro-variables in a compatible graph $\Gm$ by $V^m = \{V_1, \cdots, V_{\#V} \}$ where the indices follow a topological ordering associated with $\Gm$ (chosen arbitrarily if the ordering is not unique). We will use the same notations for any intersection or union of clusters. The cardinality of each cluster is displayed in the upper left corner of its corresponding node, as represented in Figure~\ref{fig:example_cluster_DAG:a}.

A C-DAG $\Gc$ is, by definition, derived from an ADMG over the micro-variable set \(\V^m\); however, in practical applications the true causal diagram on \(\V^m\) is typically unknown.
Then, we are interested in all the ADMGs compatible with $\Gc$.

\begin{definition}[Class of Compatible Graphs]
\label{def:equivalence_micro_admgs}
Let $\V^C$ be a partition of $\V^m$, and $\Gc$ be a mixed graph on $\mathcal{V}^C$.
We denote 
$\C(\Gc) \coloneqq \{ \Gm \mid \Gm \text{ is compatible with } \Gc\}$ the equivalence class\footnote{If we denote \( \phi(\Gm;\,\mathcal{V}^C)\)  the cluster-DAG obtained from $\Gm$ via Definition~\ref{def:Cluster_DAG}, the equivalence relation is defined by \(
  \G^m_1 \sim_{\V^C} \G^m_2
  ~\Leftrightarrow~
  \phi(\G^m_1;\,\V^C)
  = 
  \phi(\G^m_2;\,\V^C)
\).} of  graphs compatible with $\G^C$.
\end{definition}

A C-DAG is a valid causal abstraction of any underlying causal diagram on micro-variables if and only if
it contains no directed cycle composed entirely of singleton clusters, which ensure the existence of at least one ADMG compatible with the mixed graph (see Proposition~\ref{prop:valid_gc} in Appendix).

\section{A Causal Calculus for Cyclic C-DAGs}
\label{sec:calculus}
We now introduce an atomically complete calculus for reasoning about cluster queries in C-DAGs. For each rule of Pearl's calculus, Theorem \ref{th:calculus} gives a graphical criterion that is sound and complete (Theorem \ref{th:atomic_completeness}): if the separation criterion for a given rule applies, then the rule is valid in all compatible graphs; if not, then there is at least one compatible graph in which the rule fails. 
Our work is constructive: if such a graph exists, then our criterion enables its construction. To establish Theorems \ref{th:calculus} and \ref{th:atomic_completeness}, we first introduce the concept of \textit{structure of interest}, then describe the associated graphs needed to define efficiently our calculus, and finally define the corresponding calculus in the third subsection.

\subsection{Structure of interest} 

Demonstrating that a rule of Pearl’s calculus fails on a compatible graph (at least one) requires exhibiting a graph on micro-variables in which the corresponding d-separation fails. Concretely, this involves the three following steps on the micro-variables, for a given C-DAG: (i) find a path connecting variables; (ii) for each collider on that path, provide a directed path to a conditioning variable; and (iii) ensure all these paths coexist in a single compatible graph. 
While (i)–(ii) could be decided via a graphical test on the C-DAG, 
step (iii) is nontrivial since paths may conflict and form cycles. To avoid this, we directly look at 
\textit{structures of interest}, which correspond to paths to which we add for each collider a directed path to a conditioning variable. Thus, we only need to test that this structure of interest connects two sets of variables. 

Definitions~\ref{def:structure_of_interest} and~\ref{def:connecting_structure_of_interest} refine and formalize this intuition.

\begin{definition}
    \label{def:structure_of_interest}
    A \emph{structure of interest} $\sigma$ is an ADMG, with a single connected component, in which each node $V$ satisfies the following property:
    \begin{itemize}
        \item $V$ has at most one outgoing arrow, or,
        \item $V$ has two outgoing arrows but no incoming arrow.
    \end{itemize}
\end{definition}

In an ADMG, executing a breadth-first search from the root set against the arrow orientation produces a subgraph that, by construction, satisfies the first condition of Definition \ref{def:structure_of_interest}. Moreover, by enforcing the second condition at each exploration step, one obtains an efficient procedure for constructing the desired structures of interest within the ADMG.

Figure~\ref{fig:arrows_in_structure_of_interest} in Appendix shows how arrows look like around a vertex in a structure of interest. We draw in Figure~\ref{fig:example_cluster_DAG:b} in bold an example of a structure of interest in a graph on micro-variables. 

\begin{definition}[Connecting structure of interest]
\label{def:connecting_structure_of_interest}
    Let $\G=(\V,\E)$ be a mixed graph. Let $\X, \Y, \Z$ be pairwise disjoint subsets of $\V$. We say that a structure of interest $\sigma  \subseteq \G$ \emph{connects} $\X$ and $\Y$ under $\Z$ and we write $\X \notind_{\sigma} \Y \mid \Z$ if the following conditions hold:
    \begin{itemize}
        \item $\X \cap \sigma \neq \emptyset$ and $\Y \cap \sigma \neq \emptyset$ \hfill \textit{($\sigma$ connects $\X$ and $\Y$)}
        \item $\Root(\sigma) \subseteq \Z \cup \X \cup \Y$ and  \hfill \textit{(all vertices of $\sigma$ are ancestors of $\Z \cup \X \cup \Y$)}
        \item $(\sigma \setminus \Root(\sigma)) \cap \Z = \emptyset$ \hfill \textit{(neither chains nor forks of $\sigma$ are in $\Z$)}
    \end{itemize}
\end{definition}
\begin{example}
    Let us consider the graph $\Gm$ and the structure of interest $\sigma^m$ depicted in Figure~\ref{fig:example_cluster_DAG:b}. The roots of $\sigma^m$ are $\Root(\sigma^m) = \{D_1, C_1, Y_1\}$. According to Definition  \ref{def:connecting_structure_of_interest},   $\sigma^m$ connects $X_1$ and $Y_1$ under $C^m \cup D^m = \{C_1, C_2, D_1 \}$. Indeed, $\sigma^m$ contains the path $\pi^m = \langle X_1,A_1,B_1,C_1,Y_1 \rangle$ which d-connects $X_1$ and $Y_1$ under $C^m \cup D^m$ in $\Gm$.
\end{example}

We remark that a path is always a structure of interest, but it may not be a connecting structure of interest. Theorem \ref{th:new_d_sep} shows that there exists a d-connecting path if and only if there exists a connecting structure of interest. 

\begin{restatable}[D-connection with structures of interests]{theorem}{thnewdsep}
\label{th:new_d_sep}
    Let $\G$ be an ADMG. Let $\X, \Y, \Z$ be pairwise disjoint subsets of nodes of $\G$. The following properties are equivalent:
    \begin{enumerate}
        \item $\X \notind_{\G} \Y \mid \Z$. \label{th:new_d_sep:1}
        \item $\G$ contains a structure of interest $\sigma$ such that $\X \notind_{\sigma} \Y \mid \Z$. \label{th:new_d_sep:2}
    \end{enumerate}
\end{restatable}

Theorem~\ref{th:new_d_sep} reduces the problem of determining whether a rule from Pearl’s calculus fails for a given C-DAG to the search for a compatible graph which contains a \textit{structure of interest} that violates the corresponding d-separation.

\subsection{Associated graphs}

\begin{figure}[t]
    \begin{subfigure}{0.3\textwidth}
        \centering
        \begin{tikzpicture}[->,>=stealth, scale = 1.5]
            \node (A) at (0,0) {${}^3A$};
            \node (B) at (1,0) {${}^2B$};
            \node (blank) at (0, -.866)  {\phantom{${}^1X$}};
            \path (A) edge [bend left = 35] (B);
            \path (B) edge [bend left = 35] (A);
            \path (A) edge [->, loop left, looseness=5, in=155, out=200, min distance=0.5cm] node {} (A);
            \path (B) edge [->, loop right, looseness=5, in=20, out=65, min distance=0.5cm, draw=none] node {} (B); 
        \end{tikzpicture}
        \caption{}
        \label{fig:example_gcu:1}
    \end{subfigure}
    \hfill
    \begin{subfigure}{0.3\textwidth}
        \centering
        \begin{tikzpicture}[->,>=stealth, scale = 1.5]
            \node (A) at (0,0) {${}^1A$};
            \node (B) at (1,0) {${}^2B$};
            \node (C) at (2,0) {${}^1C$};
            \node (blank2) at (0, -.866)  {\phantom{${}^1X$}};
            \path (A) edge [bend left = 35] (B);
            \path (B) edge [bend left = 35] (A);
            \path (B) edge [bend left = 35] (C);
            \path (C) edge [bend left = 35] (B);
        \end{tikzpicture}
        \caption{}
        \label{fig:example_gcu:2}
    \end{subfigure}
    \hfill
    \begin{subfigure}{0.3\textwidth}
        \centering
        \begin{tikzpicture}[->, >=stealth, scale=1.5]
            \node (A) at (0, 0) {${}^1A$};
            \node (X) at (.5, -.866) {${}^1X$};
            \node (Y) at (.5, 0.866) {${}^1Y$};
            \node (B) at (1, 0) {${}^3B$};
            \node (Z) at (2, 0) {${}^1Z$};

            \draw (Y) edge (A);
            \draw (X) edge (A);
            \draw (A) edge (B);
            \draw (B) edge (Z);
            \path (Z) edge[bend right= 35] (A);
            \draw[<->, dashed] (B) to[bend right=35] (Z);
        \end{tikzpicture}
        \caption{}
        \label{fig:example_gcu:3}
    \end{subfigure}

    \begin{subfigure}{0.3\textwidth}
        \centering
        \begin{tikzpicture}[->,>=stealth, scale = 1.5]
            \node (A1) at (0,1) {$A_1$};
            \node (A2) at (0,0) {$A_2$};
            \node (A3) at (0,-1) {$A_3$};
            \node (B1) at (1,.5) {$B_1$};
            \node (B2) at (1,-.5) {$B_2$};
            \path (A1) edge  (A2);
            \path (A2) edge  (A3);
            \path (A1) edge [bend right = 35]  (A3);
            \path (A1) edge  (B2);
            \path (B1) edge  (A3);
            
            \tikzset{tochoose/.style={->,dotted, thick}}
            
            \path (A1) edge [bend left=20, tochoose] (B1);
            \path (B1) edge [bend left=20, tochoose] (A1);
            \path (B1) edge [bend left=20, tochoose] (A2);
            \path (A2) edge [bend left=20, tochoose] (B1);
            \path (A2) edge [bend left=20, tochoose] (B2);
            \path (B2) edge [bend left=20, tochoose] (A2);
            \path (B2) edge [bend left=20, tochoose] (A3);
            \path (A3) edge [bend left=20, tochoose] (B2);
        \end{tikzpicture}
        \caption{}
        \label{fig:example_gcu:4}
    \end{subfigure}
    \hfill
    \begin{subfigure}{0.3\textwidth}
        \centering
        \begin{tikzpicture}[->, >=stealth, scale = 1.5]
            \node (A1) at (0,0) {$A_1$};
            \node (B1) at (1,.5) {$B_1$};
            \node (B2) at (1,-.5) {$B_2$};
            \node (C1) at (2,0) {$C_1$};
            \node (blank1) at (0,1)  {}; 
            \node (blank2) at (0,-1)  {\phantom{$X_1$}};
            
            \path (A1) edge  (B2);
            \path (B1) edge  (A1);
            \path (B1) edge  (C1);
            \path (C1) edge  (B2);

        \end{tikzpicture}
        \caption{}
        \label{fig:example_gcu:5}
    \end{subfigure}
    \hfill
    \begin{subfigure}{0.3\textwidth}
        \centering
        \begin{tikzpicture}[->, >=stealth, xscale = 1.5, yscale=1.5]
            \node (A1) at (0, 0) {$A_1$};
            \node (X1) at (0, -1) {$X_1$};
            \node (Y1) at (0, 1) {$Y_1$};
            \node (B1) at (1,1) {$B_1$};
            \node (B2) at (1,0) {$B_2$};
            \node (B3) at (1,-1) {$B_3$};
            \node (Z1) at (2,0) {$Z_1$};

            \draw (Y1) edge (A1);
            \draw (X1) edge (A1);
            \draw (A1) edge (B3);
            \draw (B1) edge (Z1);
            \path (Z1) edge[bend right= 35] (A1);
            \draw[<->, dashed] (B1) to[bend left=35] (Z1);
            \draw[<->, dashed] (B2) to[bend right=30] (Z1);
            \draw[<->, dashed] (B3) to[bend right=35] (Z1);

            \draw[tochoose] (A1) edge (B2);
            \draw[tochoose] (B2) edge (Z1);
        \end{tikzpicture}
        \caption{}
        \label{fig:example_gcu:6}
    \end{subfigure}
    \caption{On the first row (Figures \ref{fig:example_gcu:1}, \ref{fig:example_gcu:2} and \ref{fig:example_gcu:3}), three examples of C-DAG are given. On the second row (respectively, Figures \ref{fig:example_gcu:4}, \ref{fig:example_gcu:5} and \ref{fig:example_gcu:6}), we represent the corresponding unfolded and canonical compatible graphs. The plain and dashed arrows corresponds to $\Gt$, whereas the dotted arrows represent the "eligible" arrows. Lemma \ref{lemma:gm_subgraph_gcu} and Figure~\ref{fig:example_gcu:5} show that there is no graph compatible with the C-DAG depicted in Figure~\ref{fig:example_gcu:2} such that $A_1$ and $C_1$ are connected by a directed path. Similarly, Proposition \ref{lemma:gm_cup_g3_compatible} and Figure~\ref{fig:example_gcu:6} show that there is no graph $\Gm$ compatible with the C-DAG depicted in Figure~\ref{fig:example_gcu:3} such that $A_1 \in \Anc(Z_1, \Gm)$. }\label{fig:example_gcu}
\end{figure}
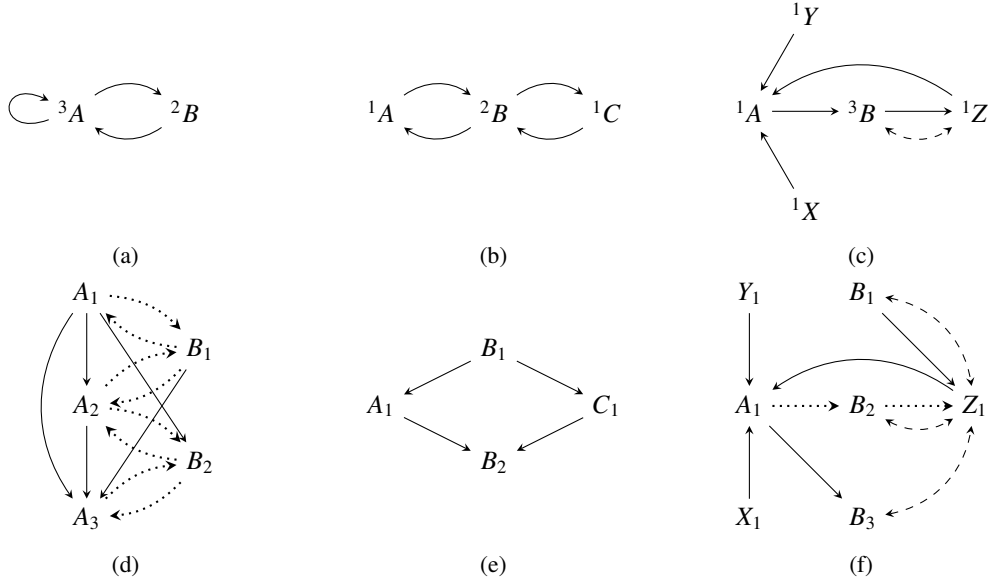

Enumerating all compatible graphs is generally infeasible due to their potentially large number. To address this challenge, we define two mixed graphs. The canonical compatible graph (Definition~\ref{def:gm_min}) allows efficient verification of whether given structures of interest exist in some compatible graph. The unfolded graph (Definition~\ref{def:unfolded_graph}) aggregates all structures of interest present in at least one compatible graph. To look for a structure of interest, we first check in the unfolded graph. However, since it may include spurious structures that do not correspond to any actual compatible graph, we afterward filter out these spurious structures from the canonical compatible graph.

\begin{definition}[Canonical Compatible Graph]
\label{def:gm_min}
    Let $\Gc$ be a C-DAG. Its corresponding \emph{canonical compatible graph} is the ADMG $\Gt = \left( \mathcal{V}^m_{\text{can}}, \mathcal{E}^m_{\text{can}}  \right)$,  where
the set of nodes is $\mathcal{V}^m_{\text{can}} := \V^m $, and the set of edges is constructed by the following procedure:
        \begin{enumerate}
            \item For all dashed-bidirected-arrows $V^C \dashleftrightarrow W^C$ in $\Gc$, add the dashed-bidirected-arrows $V_v \dashleftrightarrow W_w$ for all $v,w \in \{1,\cdots, \#V^C\} \times \{1,\cdots, \#W^C\}$ such that $V_v \neq W_w$.\label{def:gm_min:1}
            
            \item For all self-loop $\leftselfloop V^C$, add the arrow $V_i \rightarrow V_j$ for all $i,j \in \{1,\cdots, \#V^C\}^2$ such that $i <j$. \label{def:gm_min:2}
            
            \item For all arrows $V^C \rightarrow W^C$, with $V^C \neq W^C$, add the arrow $V_1 \rightarrow W_{\#W^C}$. \label{def:gm_min:3}
        \end{enumerate}
\end{definition}
Examples of canonical compatible graphs are given in Figure~\ref{fig:example_gcu}.
As stated in Proposition~\ref{prop:gm_min}, the canonical compatible graph is itself compatible and canonical in the sense that it can be added to any compatible graph without violating compatibility.

\begin{restatable}{proposition}{propgmmin}
\label{prop:gm_min}
     Let $\Gc$ be a C-DAG and $\Gt$ be its corresponding canonical compatible graph. Then, the following properties hold:
     \begin{enumerate}
         \item $\Gt \in \C(\Gc)$.
         
         \item For all $\Gm \in \C(\Gc)$,  $\Gt \cup \Gm \in \C(\Gc)$.
     \end{enumerate}
\end{restatable}

Consequently, if a structure of interest exists in some compatible graph, it must also coexist with the canonical compatible graph, meaning that its addition to the canonical compatible graph does not create cycles\footnote{Let $\sigma^m \subseteq \Gm$. Since $\Gm \cup \Gt$ is acyclic, $\sigma^m\cup \Gt$ is also acyclic.}.

We now introduce the unfolded graph. 

\begin{definition}[Unfolded graph]
\label{def:unfolded_graph}
    Let $\Gc$ be a C-DAG and $\Gt =  (\mathcal{V}^m_{\text{can}}, \mathcal{E}^m_{\text{can}})$ be its corresponding canonical compatible graph. Its corresponding \emph{unfolded graph} is $\Gcu = \left( \mathcal{V}_{\text{u}}, \mathcal{E}_{\text{u}}  \right)$, the mixed graph defined by the following procedure:
    \begin{itemize}
        \item $\mathcal{V}_{\text{u}} \coloneqq  \mathcal{V}^m$.
        
        \item Let us consider the following set:
        \[
            \E_{\text{eligible}} \coloneqq \left\{
            V_v \rightarrow W_w \,\middle|\,
            \begin{cases}
                V^C \rightarrow W^C \subseteq \Gc, \text{ and,} \\
                \Gt \cup \left\{V_v \rightarrow W_w\right\}   \text{ is acyclic.} 
            \end{cases}
            \right\}
        \]

        \noindent Then $\E_u \coloneqq \E^m_{\text{can}} \cup \E_{\text{eligible}}$.
    \end{itemize}
\end{definition}
Examples of unfolded graphs are given in Figure~\ref{fig:example_gcu}. As shown in Proposition~\ref{lemma:gm_subgraph_gcu}, the unfolded graph is a supergraph of any compatible graph. Therefore, if there exists a compatible graph containing a structure of interest $\sigma^m$, it follows that $\sigma^m$ must also appear in the unfolded graph. This implies that it is no longer necessary to enumerate all compatible graphs in the search for a structure of interest; instead, it suffices to search within the unfolded graph alone.  Figure~\ref{fig:example_gcu} illustrates how the unfolded graph can be used to demonstrate the non-existence of certain structures within the set of compatible graphs.

\begin{restatable}{proposition}{propgcu}
\label{lemma:gm_subgraph_gcu}
    Let $\Gc$ be a C-DAG and $\Gcu$ be its corresponding unfolded graph. Then, any compatible graph $\Gm$ is a subgraph of $\Gcu$ up to a permutation of indices in each cluster.
\end{restatable}

The unfolded graph defines the search space for the structures of interest, while the canonical compatible graph, as stated in Proposition \ref{prop:gcugcan}, ensures that these structures can indeed be realized within a compatible graph.

\begin{restatable}{proposition}{propgcuetgcan}
\label{prop:gcugcan}
Let $\Gc$ be a C-DAG and $\Gcu$ be its corresponding unfolded graph. Let $\sigma^m$ be a structure of interest in $\Gcu$. If $\Gt \cup \sigma^m$ is acyclic, then $\Gt \cup \sigma^m$ is compatible with $\Gc$.
\end{restatable}

By keeping these two notions distinct, we are able to apply mutilations directly on the unfolded graph without restricting the overall class of compatible graphs. 

\subsection{Calculus}
It is important to note that, in general, mutilating all graphs compatible with a C-DAG yields a strictly smaller set of graphs than the set of graphs compatible with the mutilated C-DAG, as illustrated in Example~\ref{ex:commute_pas}.\footnote{A similar phenomenon was observed by \cite{zhang_causal_2008} in the context of ancestral graphs.} This means that one cannot do the  do-calculus on the class of graphs defined by the mutilated C-DAG.

\begin{example}
\label{ex:commute_pas}
Let us consider the C-DAG $\Gc \coloneqq {}^2A \rightleftarrows {}^1B$. We have the following identities:
\begin{itemize}
    \item $\left\{\Gm_{\underline{B_1}} \mid \Gm \in \C(\Gc) \right\} = \left\{ 
    \begin{tikzpicture}[scale=0.4, baseline={(current bounding box.center)}]
        \node (A1) at (0,1) {$A_1$};
        \node (A2) at (0,-1) {$A_2$};
        \node (B1) at (2,0) {$B_1$};
        \draw[->] (A1) -- (B1);
    \end{tikzpicture}, \quad
    \begin{tikzpicture}[scale=0.5, baseline={(current bounding box.center)}]
        \node (A1) at (0,1) {$A_1$};
        \node (A2) at (0,-1) {$A_2$};
        \node (B1) at (2,0) {$B_1$};
        \draw[->] (A2) -- (B1);
    \end{tikzpicture}
    \right\}$
    
    \item $\C\left(\Gc_{\underline{B^C}}\right) = \left\{
    \begin{tikzpicture}[scale=0.4, baseline={(current bounding box.center)}]
        \node (A1) at (0,1) {$A_1$};
        \node (A2) at (0,-1) {$A_2$};
        \node (B1) at (2,0) {$B_1$};
        \draw[->] (A1) -- (B1);
    \end{tikzpicture}, \quad
    \begin{tikzpicture}[scale=0.5, baseline={(current bounding box.center)}]
        \node (A1) at (0,1) {$A_1$};
        \node (A2) at (0,-1) {$A_2$};
        \node (B1) at (2,0) {$B_1$};
        \draw[->] (A2) -- (B1);
    \end{tikzpicture}, \quad
    \begin{tikzpicture}[scale=0.5, baseline={(current bounding box.center)}]
        \node (A1) at (0,1) {$A_1$};
        \node (A2) at (0,-1) {$A_2$};
        \node (B1) at (2,0) {$B_1$};
        \draw[->] (A1) -- (B1);
        \draw[->] (A2) -- (B1);
    \end{tikzpicture}
    \right\}$
\end{itemize}

And then, $\left\{\Gm_{\underline{B_1}} \mid \Gm \in \C(\Gc) \right\} \subsetneqq \C\left(\Gc_{\underline{B^C}}\right) $.
\end{example}

However, thanks to the unfolded graph and the canonical compatible graph, the rules of do-calculus for cluster queries can be encoded in a sound and complete manner. This principle is formally established in Theorems~\ref{th:calculus} and~\ref{th:atomic_completeness}. An illustrative application of Theorem~\ref{th:calculus} is provided in Example \ref{example:calculus} (a complementary example is provided in Figure~\ref{fig:th:cluster_sep_cluster_mutilation} in  Appendix).

\begin{restatable}[Calculus]{theorem}{thcalculus}
\label{th:calculus}
Let \(\Gc\) be a C-DAG and let $\Gcu$ be its corresponding unfolded graph. Let \(\X^C,\Y^C,\Z^C,\mathcal{W}^C\) be pairwise distinct subsets of nodes.  Then, for any density
$P$ induced by a SCM compatible with $\Gc$\footnote{In Pearl's framework, we would consider DAGs and any positive, compatible density.}, the following rules apply:
\begin{enumerate}
  \item[R1.] \(
      P\bigl(\mathbf{y^m} \mid \Do(\mathbf{w^m}),\mathbf{x^m}, \mathbf{z^m}) = P\bigl(\mathbf{y^m} \mid \Do(\mathbf{w^m}),\mathbf{z^m})
    \)
     if $\Gcu_{\overline{\W^m}}$ does not contain a structure of interest $\sigma^m$ such that $ \X^m \notind_{\sigma^m} \Y^m \mid \W^m,\Z^m$ and $\Gt \cup \sigma^m$ is acyclic.

  \item[R2.] \(
      P\bigl(\mathbf{y^m} \mid \Do(\mathbf{w^m}), \Do(\mathbf{x^m}), \mathbf{z^m}) = P\bigl(\mathbf{y^m} \mid \Do(\mathbf{w^m}),\mathbf{x^m}, \mathbf{z^m})
    \)
    if $\Gcu_{\overline{\W^m},\underline{\X^m}}$ does not contain a structure of interest $\sigma^m$ such that $\X^m \notind_{\sigma^m} \Y^m \mid \W^m,\Z^m$ and  $\Gt \cup \sigma^m$ is acyclic.

     \item[R3.] \(
      P\bigl(\mathbf{y^m} \mid \Do(\mathbf{w^m}), \Do(\mathbf{x^m}), \mathbf{z^m}) = P\bigl(\mathbf{y^m} \mid \Do(\mathbf{w^m}),\mathbf{z^m})
    \) if \( {\Gcu}_{\overline{\W^m}} \) does not contain a structure of interest $\sigma^m$ such that $ \X^m \notind_{\sigma^m} \Y^m \mid \W^m,\Z^m$, $\Gt \cup \sigma^m $ is acyclic and $\Root(\sigma^m) \subseteq (\W^m \cup \Z^m) \cup \Y^m$.
    
\end{enumerate}
 \end{restatable}

The first two rules of Theorem \ref{th:calculus} are very similar to the first two rules of Pearl’s do-calculus. In contrast, the third rule of Pearl’s do-calculus requires verifying the d-separation condition  \(
\Y^m \ind_{\Gm_{\overline{\W^m}, \overline{\X^m (\Z^m)}}} X^m \mid \W^m, \Z^m,
\) in all compatible graph $\Gm$, where \(\X^m(\Z^m) = X^m \setminus \Anc(\Z^m,\Gm_{\overline{\W^m}})\). Since \(\X^m(\Z^m)\) is not, in general, a union of clusters, the associated mutilation depends on the particular graph $\Gm$. As a result, it is not possible to directly apply this mutilation to the unfolded graph to derive an atomically complete criterion for Rule~3. 

Nonetheless, if Rule~3 does not hold in some compatible graph $\Gm$, then there exists a structure of interest between \(\Y^m\) and \(X^m\) in \(\Gm_{\overline{\W^m}, \overline{\X^m (\Z^m)}}\). If this structure includes a root \(X_x \in \X^m\), then \(X_x\) must be an ancestor of some \(Z_z \in \Z^m\) in the mutilated graph. In such a case, we can augment the structure of interest by explicitly adding the directed path from \(X_x\) to \(\Z^m\), resulting in a new structure whose roots lie outside \(\X^m\). This constructive process of eliminating roots from \(\X^m\) is behind the third rule in Theorem~\ref{th:calculus}.

For all the rules, structures of interest are sought in the unfolded graph. The canonical compatible graph is then used to ensure that the identified structure of interest actually exists in a compatible graph.

All the calculus rules given in Theorem \ref{th:calculus} are atomatically complete, as stated by Theorem \ref{th:atomic_completeness}. 

\begin{restatable}[Atomic completeness]{theorem}{thatomiccompleteness}
\label{th:atomic_completeness}
The calculus in Theorem \ref{th:calculus} is atomically complete i.e. if the rule does not hold given a C-DAG, then there exists a compatible graph in which the corresponding rule in Pearl's calculus fails.
\end{restatable}

\begin{example}
\label{example:calculus}
    Let us consider $\Gc$, the C-DAG depicted in Figure~\ref{fig:example_gcu:3}. Figure~\ref{fig:example_gcu:6} displays the corresponding unfolded graph and canonical compatible graphs. The plain and dashed arrows represent~$\Gt$, while the dotted arrows denote the "eligible" edges. According to the second rule of Theorem~\ref{th:calculus}, we have $P(\mathbf{y^m}\mid \Do(\mathbf{z^m})) = P(\mathbf{y^m}\mid \mathbf{z^m})$. Indeed any structure of interest $\sigma^m$ which connects $Y^m$ and $Z^m$ under $\emptyset$ contains the arrows $A_1 \rightarrow B_2 \rightarrow Z_1$. Since $\Gt$ contains $Z_1 \rightarrow A_1$, we know that $\Gt \cup \sigma^m$ contains a cycle. Therefore, $\Gcu_{\overline{\emptyset},\underline{Z^m}}$ does not contain a structure of interest $\sigma^m$ that connects $Y^m$ and $Z^m$ under $\emptyset$ such that $\Gt \cup \sigma^m$ is acyclic.
\end{example}

\section{Computational efficiency: reducing clusters of large size to 3 nodes}
\label{sec:3nodes}

While Theorem~\ref{th:calculus} provides a sound and atomically complete calculus, its direct application may be impractical for large clusters, as computing $\Gcu$ becomes intractable due to a combinatorial explosion in the number of edges.
To address this, we associate with any C-DAG $\Gc$ a simplified C-DAG $\G^C_{\leq3}$ on the same set of nodes (but with different cardinals), where each cluster of size greater than 3 is reduced to size 3. The set of edges of $\G^C_{\leq3}$ is the set of edges of $\G^C$. The key difference is that $\mathcal{C}(\G^C_{\leq3}) \neq \mathcal{C}(\G^C)$, because the graphs compatible with $\G^C_{\leq3}$ contain fewer nodes and different edges than the graphs compatible with $\G^C$. We illustrate this in Figure~\ref{fig:example}. 
Notably, Theorem~\ref{th:infinity_leq_three} shows that applying the calculus on $\Gc$ or on $\G^C_{\leq3}$ leads to the same results.

\begin{restatable}[Infinity is at most three]{theorem}{thinfinityleqthree}
\label{th:infinity_leq_three}
    Let $\Gc$ be a C-DAG and $\G^C_{\leq3}$ be the corresponding C-DAG where all clusters of size greater than 3 are reduced to size 3. Let $\W^C$, $\X^C$, $\Y^C$ and $\Z^C$ be pairwise disjoint subsets of nodes. For \(i\in\{1,2,3\}\), let \(R_i(\W, \X,\Y,\Z)\) be the \(i\)\textsuperscript{th} rule of Pearl's Calculus applied to \((\W,\X,\Y,\Z)\), and say it "does not holds in \(\G\)” whenever its associated d-separation condition in the associated mutilated graph is not satisfied. The following propositions are equivalent:
    \begin{enumerate}
      \item There exists $\G^m \in \C(\Gc)$ in which $R_i(\W^m,\X^m,\Y^m,\Z^m)$ does not hold. \label{th:infinity_leq_three:1}
      \item There exists $\G^m_{\le3} \in \C(\G^C_{\le3})$ in which $R_i(\W_{\le3}^m,\X^m_{\le3},\Y^m_{\le3},\Z^m_{\le3})$ does not hold. \label{th:infinity_leq_three:2}
    \end{enumerate}
\noindent Where $\W_{\le3}^m$, $\X^m _{\leq3}$, $\Y^m _{\leq3}$ and $\Z^m _{\leq3}$ are the sets of nodes corresponding to $\W^C$, $\X^C$, $\Y^C$ and $\Z^C$ in $\G^C_{\leq3}$
\end{restatable}

Figure~\ref{fig:example} illustrates Theorem~\ref{th:infinity_leq_three} by showing how a graph \(\G^m \in \C(\Gc)\), in which a given d‑separation does not hold, is transformed into \(\G^m_{\le3} \in \C(\G^C_{\le3})\), in which the corresponding d‑separation does not hold as well. The figure highlights how this transformation impacts the structure of interest that violates the d‑separation by omitting all irrelevant dependencies (see Figure~\ref{fig:example:2}).

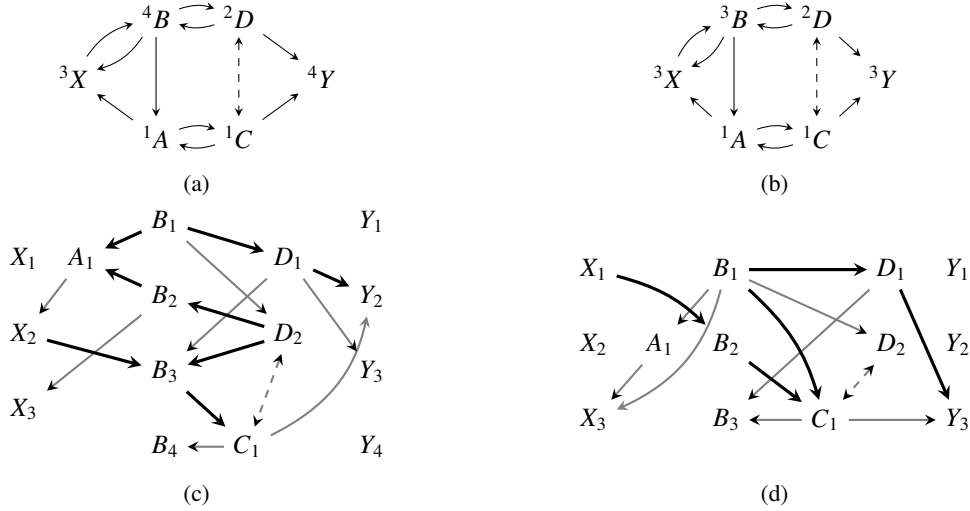
\begin{figure}[t]
    \centering 
   \begin{subfigure}[t]{0.45\textwidth}
    \centering
    \begin{tikzpicture}[->, >=stealth, xscale=1.1]
        \def\xX{0} \def\xA{1} \def\xB{1} \def\xC{2} \def\xD{2} \def\xY{3}

        \node[text=black] (X) at (\xX, -0.8) {${}^3X$};
        \node[text=black] (A) at (\xA, -1.6) {${}^1A$};
        \node[text=black] (B) at (\xB, 0) {${}^4B$};
        \node[text=black] (C) at (\xC, -1.6) {${}^1C$};
        \node[text=black] (D) at (\xD, 0) {${}^2D$};
        \node[text=black] (Y) at (\xY, -0.8) {${}^4Y$};

        \path (D) edge [->, bend left=20] (B);
\path (B) edge [->, bend left=20] (D); 
\path (X) edge [->, bend left=20] (B);
\path (B) edge [->, bend left=20] (X);        \draw[->] (A) -- (X);
        \draw[->] (B) -- (A);
        \draw[<->, dashed] (D) -- (C);
\path (A) edge [->, bend left=20] (C);
\path (C) edge [->, bend left=20] (A);  
        \draw[->] (D) -- (Y);
        \draw[->] (C) -- (Y);
    \end{tikzpicture}
    \caption{}
    \label{fig:example:01}
\end{subfigure}
\hfill
   \begin{subfigure}[t]{0.45\textwidth}
    \centering
    \begin{tikzpicture}[->, >=stealth, xscale=1.1]
        \def\xX{0.2} \def\xA{1} \def\xB{1} \def\xC{2} \def\xD{2} \def\xY{2.8}

        \node[text=black] (X) at (\xX, -0.8) {${}^3X$};
        \node[text=black] (A) at (\xA, -1.6) {${}^1A$};
        \node[text=black] (B) at (\xB, 0) {${}^3B$};
        \node[text=black] (C) at (\xC, -1.6) {${}^1C$};
        \node[text=black] (D) at (\xD, 0) {${}^2D$};
        \node[text=black] (Y) at (\xY, -0.8) {${}^3Y$};

        \path (D) edge [->, bend left=20] (B);
\path (B) edge [->, bend left=20] (D); 
\path (X) edge [->, bend left=20] (B);
\path (B) edge [->, bend left=20] (X);        \draw[->] (A) -- (X);
        \draw[->] (B) -- (A);
        \draw[<->, dashed] (D) -- (C);
\path (A) edge [->, bend left=20] (C);
\path (C) edge [->, bend left=20] (A);  
        \draw[->] (D) -- (Y);
        \draw[->] (C) -- (Y);
    \end{tikzpicture}
    \caption{}
    \label{fig:example:02}
\end{subfigure}

    \begin{subfigure}[t]{0.45\textwidth}
        \centering
        \begin{tikzpicture}[->, >=stealth, xscale=1.1]
            \def\xX{0.3} \def\xA{1} \def\xB{2} \def\xC{3} \def\xD{3.5} \def\xY{4.5}

            \foreach \i in {1,...,3} {
                \node[text=black] (X\i) at (\xX, -0.5 -\i+1) {$X_{\i}$};
            }
            \node[text=black] (A1) at (\xA, -0.5) {$A_1$};
            \foreach \i in {1,...,4} {
                \node[text=black] (B\i) at (\xB, -\i+1) {$B_{\i}$};
            }
            \node[text=black] (C1) at (\xC, -3) {$C_1$};
            \foreach \i in {1,...,2} {
                \node[text=black] (D\i) at (\xD, -0.5 -\i+1) {$D_{\i}$};
            }
            \foreach \i in {1,...,4} {
                \node[text=black] (Y\i) at (\xY, -\i+1) {$Y_{\i}$};
            }

            \draw[->, draw=gray, line width=0.8pt] (C1) -- (B4);
            \draw[->, draw=gray, line width=0.8pt] (A1) -- (X2);
            \draw[->, draw=gray, line width=0.8pt, bend right=25] (C1) edge[gray] (Y2);
            \draw[<->, dashed, draw=gray, line width=0.8pt] (D2) -- (C1);
            \draw[->, draw=gray, line width=0.8pt] (B2) -- (X3);
            \draw[->, draw=gray, line width=0.8pt] (B1) -- (D2);
            \draw[->, draw=gray, line width=0.8pt] (D1) -- (B3);
            \draw[->, draw=gray, line width=0.8pt] (D1) -- (Y3);

            \draw[->, line width=1.2pt] (X2) -- (B3);
            \draw[->, line width=1.2pt] (D2) -- (B3);
            \draw[->, line width=1.2pt] (D2) -- (B2);
            \draw[->, line width=1.2pt] (B2) -- (A1);
            \draw[->, line width=1.2pt] (B1) -- (A1);
            \draw[->, line width=1.2pt] (B1) -- (D1);
            \draw[->, line width=1.2pt] (D1) -- (Y2);
            \draw[->, line width=1.2pt] (B3) -- (C1);
        \end{tikzpicture}
        \caption{}
        \label{fig:example:1}
    \end{subfigure}
\hfill
\begin{subfigure}[t]{0.45\textwidth}
        \centering
        \begin{tikzpicture}[->, >=stealth, xscale=1.1]
            \def\xX{0} \def\xA{0.8} \def\xB{1.6} \def\xC{2.8} \def\xD{3.6} \def\xY{4.4}

            \foreach \i in {1,...,3} {
                \node[text=black] (X\i) at (\xX, -\i+1) {$X_{\i}$};
            }
            \node[text=black] (A1) at (\xA, -1) {$A_1$};
            \foreach \i in {1,...,3} {
                \node[text=black] (B\i) at (\xB, -\i+1) {$B_{\i}$};
            }
            \node[text=black] (C1) at (\xC, -2) {$C_1$};
            \foreach \i in {1,...,2} {
                \node[text=black] (D\i) at (\xD, -\i+1) {$D_{\i}$};
            }
            \foreach \i in {1,...,3} {
                \node[text=black] (Y\i) at (\xY, -\i+1) {$Y_{\i}$};
            }

            \node at (0,-2.5) {};

            \draw[->, draw=gray, line width=0.8pt] (C1) -- (B3);
            \draw[->, draw=gray, line width=0.8pt] (A1) -- (X3);
            \draw[->, draw=gray, line width=0.8pt] (C1) -- (Y3);
            \draw[<->,dashed, draw=gray, line width=0.8pt] (D2) -- (C1);
            \draw[->, draw=gray, line width=0.8pt, bend left=25] (B1) edge[gray] (X3);
            \draw[->, draw=gray, line width=0.8pt] (B1) -- (D2);
            \draw[->, draw=gray, line width=0.8pt] (D1) -- (B3);
            \draw[->, draw=gray, line width=0.8pt] (B1) -- (A1);

            \path[->, line width=1.2pt] (X1) edge[bend left=15] (B2);
            \draw[->, line width=1.2pt] (B1) -- (D1);
            \draw[->, line width=1.2pt] (D1) -- (Y3);
            \draw[->, line width=1.2pt] (B2) -- (C1);
            \path[->, line width=1.2pt] (B1) edge[bend left=15] (C1);
        \end{tikzpicture}
        \caption{}
        \label{fig:example:2}
    \end{subfigure}
    \caption{Top: $\Gc$ (left) and $\G^C_{\leq 3}$ (right). Bottom: a graph compatible with $\Gc$ (left) and a graph compatible with $\G^C_{\leq 3}$ (right). In Figure~\ref{fig:example:1}, the arrows in \textbf{bold black} represent a structure of interest that connects $X^m$ and $Y^m$ under $C^m \cup A^m$. In Figure~\ref{fig:example:2}, the arrows in \textbf{bold black} represent the structure of interest that connects $X^m_{\leq 3}$ and $Y^m_{\leq 3}$ under $C^m_{\leq 3} \cup A^m_{\leq 3}$, which is obtained by applying the strategy used in the proof of Theorem~\ref{th:infinity_leq_three}.}
    \label{fig:example}
\end{figure}

Theorem~\ref{th:infinity_leq_three} shows that reducing cluster size to at most three preserves all relevant dependencies. This bound is tight: in some C-DAGs, any further reduction of the size of the clusters (by removing more nodes)
would necessarily lose causal information. Example~\ref{example:infty_geq_3} illustrates such a case.

\begin{example}[Infinity is at least three]
\label{example:infty_geq_3}
    Let $\Gc$ be the C-DAG defined by Figure~\ref{fig:infty_geq_3:Gc}. Let $\X^C = \{X^C\}$, $\Y^C = \{Y^C\}$ and  $\Z^C = \{{Z^1}^C, {Z^2}^C\}$. There exists a compatible graph $\Gm$ in which $\X^m \notind_{\Gm} \Y^m \mid \Z^m$: we displayed it in Figure~\ref{fig:infty_geq_3:Gm}. 
    
    There is only one graph compatible with $\G^C_{\leq 2}$ : $\C(\G^C_{\leq 2}) = \{\mathcal{G}^m_{\leq2}\}$ (the one displayed in Figure~\ref{fig:infty_geq_3:Gm_leq2}). Moreover,  in $\G^{m}_{\leq 2}$, the corresponding dependence does not hold.
\end{example}

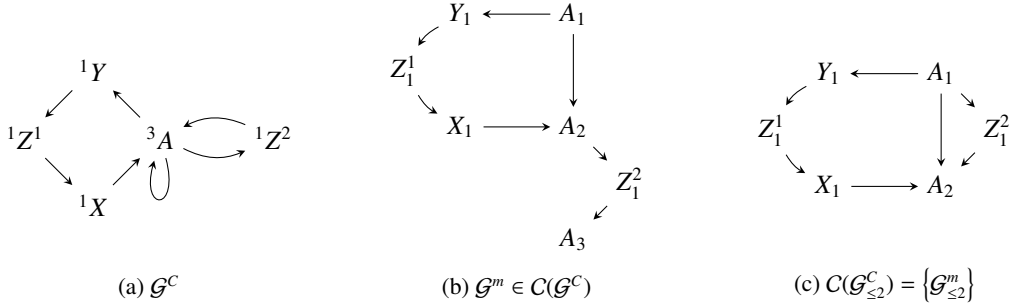
\begin{figure}[t]
    \centering
    \begin{subfigure}[b]{0.3\textwidth}
        \centering
        \begin{tikzpicture}[->, >=stealth, xscale = 1.5, yscale=1.5]
            \node (Y)  at (-0.6,.6)   {${}^1Y$};
            \node (Z1) at (-1.2,0)     {${}^1Z^1$};
            \node (X)  at (-0.6,-.6)   {${}^1X$};
            \node (A)  at (0,0)      {${}^3A$};
            \node (Z2) at (1,0)    {${}^1Z^2$};
            \node at (0,-1) {};
            \draw (Y)  edge (Z1);
            \draw (Z1) edge (X);
            \draw (X)  edge (A);
            \draw (A)  edge[bend right=25] (Z2);
            \draw (Z2) edge[bend right=25] (A);
            \draw (A)  edge[loop below] (A);
            \draw (A)  edge (Y);
        \end{tikzpicture}
        \caption{$\Gc$}
        \label{fig:infty_geq_3:Gc}
    \end{subfigure}
    %
    \hfill
    \begin{subfigure}[b]{0.3\textwidth}
        \centering
        \begin{tikzpicture}[->, >=stealth, scale=1.5]
            \node (Y)  at (-1,1)   {$Y_1$};
            \node (Z1) at (-1.5,.5)  {$Z^1_1$};
            \node (X)  at (-1,0)   {$X_1$};
            \node (A1) at (0,1)    {$A_1$};
            \node (A2) at (0,0)    {$A_2$};
            \node (A3) at (0,-1)   {$A_3$};
            \node (Z2) at (0.5,-.5)  {$Z^2_1$};
            \draw (Y)  edge[bend right=15] (Z1);
            \draw (Z1) edge[bend right=15] (X);
            \draw (X)  edge (A2);
            \draw (A2) edge (Z2);
            \draw (Z2) edge (A3);
            \draw (A1) edge (A2);
            \draw (A1) edge (Y);
        \end{tikzpicture}
        \caption{$\Gm \in \C(\Gc)$}
        \label{fig:infty_geq_3:Gm}
    \end{subfigure}
    \hfill
    \begin{subfigure}[b]{0.3\textwidth}
        \centering
        \begin{tikzpicture}[->, >=stealth, scale=1.5]
            \node (Y)  at (-1,1)   {$Y_1$};
            \node (Z1) at (-1.5,.5)  {$Z^1_1$};
            \node (X)  at (-1,0)   {$X_1$};
            \node (A1) at (0,1)    {$A_1$};
            \node (A2) at (0,0)    {$A_2$};
            \node (Z2) at (0.5,.5)   {$Z^2_1$};
            \node at (0,-.5) {};
            \draw (Y)  edge[bend right=15] (Z1);
            \draw (Z1) edge[bend right=15] (X);
            \draw (X)  edge (A2);
            \draw (A1) edge (Z2);
            \draw (Z2) edge (A2);
            \draw (A1) edge (A2);
            \draw (A1) edge (Y);
        \end{tikzpicture}
        \caption{$\C(\G^C_{\leq 2}) = \left\{\G^m_{\leq 2}\right\}$}
        \label{fig:infty_geq_3:Gm_leq2}
    \end{subfigure}
    
    \caption{Figure~\ref{fig:infty_geq_3:Gc} depicts a cluster $\Gc$. Figure~\ref{fig:infty_geq_3:Gm} illustrates a graph compatible with $\Gc$. Figure~\ref{fig:infty_geq_3:Gm_leq2} illustrates the unique graph compatible with $\G^C_{\leq2}$.}
    \label{fig:infty_geq_3}
\end{figure}

\section{Discussion}
\label{sec:discussion}

\paragraph{On Cluster d-Separation} Theorem~\ref{th:calculus} provides a sound and atomically complete calculus for causal identification. In addition, our results offer a sound and atomically complete solution to the problem of cluster d-separation. Specifically, the criterion for cluster d-separation corresponds to the first rule of Theorem~\ref{th:calculus} when taking $\W^c = \emptyset$. Furthermore, Theorem~\ref{th:cluster_sep_cluster_mutilation} in  Appendix establishes cluster d-separation under cluster-level mutilations. Our results thus encompass both the first (association) and second (intervention) rungs of Pearl's ladder of causation.

\paragraph{Recovering the results on standard (acyclic) C-DAGs} Theorem~\ref{th:calculus} recovers the results of \citet{anand_causal_2023}. When a C-DAG \(\Gc\) is acyclic, its corresponding unfolded graph $\Gcu$ is also acyclic. As a result, $\Gcu$ is a compatible graph and standard d-separation is both sound and complete in $\Gcu$.

\paragraph{A 2-step strategy} In our study, the unfolded graph defines the search space for the structures of interest, while the canonical compatible graph ensures that these structures can indeed be realized within a compatible graph. By keeping these two notions distinct, we are able to apply mutilations directly on the unfolded graph without restricting the overall class of compatible graphs. This approach resolves the non-commutativity between mutilation and enumeration of compatible graphs, since performing mutilation before enumeration generally produces a strictly larger set of graphs than enumerating first before mutilating. A similar phenomenon was observed in \citet{zhang_causal_2008} in the context of ancestral graphs.

\paragraph{Unknown size of clusters} The typical use case for C-DAGs assumes access to individual observed variables (i.e., micro-variables). Clusters are constructed by grouping  variables based on interpretability needs and domain knowledge. In such settings, causal and confounding relationships between clusters are explicitly modeled, while dependencies within clusters are left unspecified. Crucially, the specification of a C-DAG does not require the inclusion or modeling of any unobserved variables within a cluster. Unobserved variables are only explicitly modeled when they act as latent confounders between clusters, in which case bi-directed dashed edges are introduced. This means that the cardinality of each cluster corresponds to the number of observed variables it contains, a quantity that is generally known. Therefore, the assumption of known cluster cardinality reflects realistic scenarios and does not compromise the validity of our theoretical results.

Nonetheless, if the cardinality of a cluster is overestimated, then the calculus remains sound, though not necessarily complete. Conversely, if the cardinality is underestimated, the calculus is complete, but soundness is no longer guaranteed. In cases where the true cluster size is unknown, Theorem \ref{th:infinity_leq_three} ensures that assuming a cardinality of 3 yields a sound calculus.

\section{Conclusion}
\label{sec:conclusion}

We have addressed in this study the problem of identification in causal abstractions based on arbitrary clusterings of variables in ADMGs, extending the framework considered in \citet{anand_causal_2023} to abstract graphs which potentially contain self-loops and cycles between clusters. This extension is important in practice as the structure induced on clusters of variables in a given ADMG is likely to contain cycles between clusters. In this framework, we have first reformulated the notion of d-separation in an ADMG using structures of interest, a reformulation which remains faithful to the original formulation as finding a structure of interest is sufficient to d-connect two sets, and then provided a causal calculus which is both sound and atomically complete. We further showed that any cluster can be reduced to a cluster of limited size, leading to efficient calculus rules.

In the future, we aim to establish the global completeness of the calculus, as it is currently only atomically complete. We also plan to extend this work by considering micro-level interventions, \textit{i.e.}, interventions on individual variables rather than on clusters of variables, when only the C-DAG is known.

\begin{ack}
This work was partly supported by fundings from the French government, managed by the National Research Agency (ANR), under the France 2030 program, reference ANR-23-IACL-0006 and  ANR-23-PEIA-0007.
\end{ack}

\newpage

\bibliography{References}
\bibliographystyle{apalike}  


\appendix
\newpage
\tableofcontents
\newpage

\section{Glossary of Notations}
\begin{itemize}
  \item[X:] A variable.
  \item[x:] A realized value of X.
  \item[$\X$:] A set.
  \item[$\Anc(\X, \G)$:] The set of ancestors of $\X$ in $\G$ (including $\X$ itself).
  \item[$\Desc(\X, \G)$:] The set of descendants of $\X$ in $\G$ (including $\X$ itself).
  \item[$\Root(\G)$:] The set of roots of $\G$, i.e., vertices with no child.
  \item[Active vertex:] A vertex $V$ is \emph{active} on a path relative to $\Z$ if
    (i) $V$ is a collider and $V$ or one of its descendants is in $\Z$, or
    (ii) $V$ is a non-collider and $V \notin \Z$.
  \item[Active path:] A path $\pi$ is \emph{active} given $\Z$ if all vertices on $\pi$ are active relative to $\Z$.
  \item[D-separation:] Sets $\X$ and $\Y$ are d-separated by $\Z$ if every path between them is inactive given $\Z$, denoted $\X \ind_{\G} \Y \mid \Z$.
  \item[$\G_{\overline{\X}\underline{\Z}}$:] The mutilated graph obtained by removing edges with arrowheads into $\X$ and tails from $\Z$.
  \item[$\pi_{[A,B]}$:] The subpath of $\pi$ between vertices $A$ and $B$.
  \item[$\G_1 \cup \G_2$:] The union of graphs $\G_1=(\V_1,\E_1)$ and $\G_2=(\V_2,\E_2)$, defined as $(\V_1 \cup \V_2, \E_1 \cup \E_2)$.
  \item[$\mathcal{X} \cap \G$:] The set of nodes in $\G$ that belong to $\mathcal{X}$.
  \item[$\G \setminus \mathcal{X}$:] The subgraph of $\G$ obtained by removing all vertices in $\mathcal{X}$ and their incident edges.
  \item[$\V^m$:] Micro-variables.
  \item[$\V^C$:] Partition of $\V^m$.
  \item[$\Gm$:] An ADMG on micro-variables.
  \item[$\Gc$:] A C-DAG defined in Definition~\ref{def:Cluster_DAG}.
  \item[$\C(\Gc)$:] The class of compatible graphs defined in Definition~\ref{def:equivalence_micro_admgs}.
  \item[$\sigma$:] A structure of interest.
  \item[$\Gt$:] The canonical compatible graph defined in Definition~\ref{def:gm_min}.
  \item[$\Gcu$:] The unfolded graph defined in Definition~\ref{def:unfolded_graph}.
\end{itemize}

\newpage
\section{Technical Appendices and Supplementary Material}

\begin{notation}
\label{notation:base}
    Let $\Gc$ be a C-DAG, and let $V$ be a cluster in $\Gc$.  When $V$ is seen as a node of $\Gc$, $V$ will be written as $V^C$; but when $V$ is seen as a set of variables of a graph $\Gm$ on micro-variables, $V$ will be written as $V^m = \{V_1, \cdots, V_{\#V} \}$ where the indices follow a topological ordering induced by $\Gm$ (chosen arbitrarily if the ordering is not unique).
    We will use the same notations for any intersection or union of cluster.
\end{notation}

\paragraph{Terminolgy on paths.} Let $\pi$ be a path and $\X$ a subset of variables. We say that $\pi$ intersects or encounters $\X$ if they share at least one common vertex. We treat paths as ordered lists of variables, which allows us to define the first and last encounter of $\pi$ with $\X$. The first encounter is the first vertex in the order of $\pi$ that also belongs to $\X$. Similarly, the last encounter is the last such vertex in the order of $\pi$. Let $\pi$ be a path from $\X$ to $\Y$. Let $A$ be a vertex on $\pi$. We denote by $\pi_{[\X,A]}$ the subpath of $\pi$ from its first vertex to $A$ and $\pi_{[A,\Y]}$ from A to its last vertex.

\subsection{Basic Properties of C-DAGs}

In this section, following the notations of \cite{perkovic_complete_2018}, an arrow ($\leftcomingarrow$) represents either a directed arrow ($\leftarrow$) or a dashed-bidirected arrows ($\dashleftrightarrow$).

\begin{restatable}{proposition}{propvalidgc}
\label{prop:valid_gc}
    Let $\V^C$ be a partition of $\V^m$. Let $\Gc$ be a mixed graph over $\V^C$ and let $\C\left(\Gc\right)$ denote the class of graphs compatible with $\Gc$. Then the following propositions are equivalent:
    \begin{itemize}
        \item $\C\left(\Gc\right) \neq \emptyset$ \label{prop:valid_gc:1} 
        \item $\Gc$ does not contain any cycle on clusters of size $1$. \label{prop:valid_gc:2}
    \end{itemize}
\end{restatable}

\begin{proof}
    Let us prove the two implications:
    \begin{itemize}
        \item $\Rightarrow$: If $\Gc$ contains the cycle ${}^1A \rightarrow \cdots \rightarrow {}^1A$, then any compatible graph would contain the cycle $A_1 \rightarrow \cdots \rightarrow A_1$, which is not allowed because compatible graphs have to be acyclic.
        \item $\Leftarrow$: If $\Gc$ does not contain any cycle on cluster of size $1$. Let us construct a compatible graph. For all cluster $V$ of size $1$, we put all incoming and outgoing edges at $V_1$. This does not create a cycle because, otherwise, $\Gc$ would contain a cycle on cluster of size $1$. For all other clusters $V$, as they are at least of size $2$, we can deal with $V_1$ and $V_2$. We put all outgoing edges at $V_1$ and all incoming edges at $V_2$. This construction cannot introduce a directed cycle, since no vertex in $V^m$ ever has both an incoming and an outgoing edge.
    \end{itemize}
\end{proof}

\begin{proposition}
\label{prop:remove_arrow}
    Let $\Gc$ be a C-DAG, $\Gm$ be a compatible graph with $\Gc$ and $V^C$ and $W^C$ be nodes of $\Gc$. If $\Gm$ contains two similar (same type) arrows between $V^m$ and $W^m$, then removing one of these arrows create another compatible graph.
\end{proposition}

\begin{proof}
    By definition, only one arrow between $V^m$ and $W^m$ is necessary.
\end{proof}

\begin{proposition}
\label{prop:move_arrow_up}
    Let $\Gc$ be a C-DAG, $\Gm$ be a compatible graph with $\Gc$ and $V^C$ be a node in $\Gc$, i.e. a cluster. If there exists $(i,j)$ with $i>j$ such that $\Gm$ contains the arrow $V_i \rightarrow W_w$, then there exists a compatible graph $\Gm'$ that contains the arrow $V_j \rightarrow W_w$ and not $V_i \rightarrow W_w$ if desired.
\end{proposition} 

\begin{proof}
    If $i > j$, thus by the convention of Notation~\ref{notation:base}, we know that $V_j$ is before $V_i$ in a topological order of $\Gm$. Thus, $V_j$ is not a descendant of $V_i$ in $\Gm$. Thus, $V_j$ is not a descendant of $W_w$ in $\Gm$. Therefore, adding the arrow $V_j \rightarrow W_w$ into $\Gm$ does not create a cycle, because otherwise $V_j$ would be a descendant of $W_w$ in $\Gm$.
    
    By Proposition~\ref{prop:remove_arrow}, $V_i \rightarrow W_w$ can be thereafter removed if desired.
\end{proof}

\begin{corollary}
\label{cor:move_fork_up}
    Let $\Gc$ be a C-DAG, $\Gm$ be a compatible graph and $V^C$ be a cluster. If $\Gm$ contains a path which contains a fork on $V_i$ with $i>j$, then there exists a compatible graph $\Gm'$ which contains the same path except that the fork is on $V_j$.
\end{corollary}

\begin{proof}
    Apply Proposition~\ref{prop:move_arrow_up} twice. 
\end{proof}

\begin{proposition}
\label{prop:move_arrow_down}
    Let $\Gc$ be a C-DAG, $\Gm$ be a compatible graph and $V^C$ be a cluster. If there exists $i < j$ such that $\Gm$ contains the arrow $V_i \leftcomingarrow W_w$, where  $\leftcomingarrow$, then there exists a compatible graph $\Gm'$ that contains the arrow $V_{j} \leftcomingarrow W_w$ and not $V_i \leftcomingarrow W_w$ if desired.
\end{proposition} 

\begin{proof}
    $i < j$, thus by the convention of Notation~\ref{notation:base}, we know that $W_w$ is not a descendant of $V_j$. Therefore, adding the arrow $V_j \leftcomingarrow W_w$ does not create a cycle.
    
    By Proposition~\ref{prop:remove_arrow}, $V_i \leftcomingarrow W_w$ can be removed if desired.
\end{proof}

\subsection{Proof of Theorem \ref{th:new_d_sep}}

We introduce Figure~\ref{fig:arrows_in_structure_of_interest} which helps the reader understanding structures of interest.

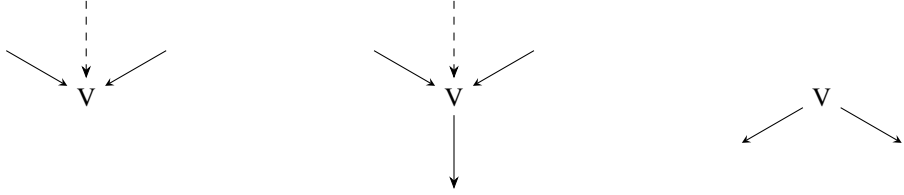
\begin{figure}[ht]
    \centering
    \begin{subfigure}[b]{0.3\textwidth}
        \centering
        \begin{tikzpicture}[->, >=Stealth, scale =.7]
            \node (A) {V};
            \node[circle, draw = none] (Blanc) at (0, -2) {};
            \foreach \i in {5,7} {
                \node[circle, draw = none] (B\i) at (90+\i*360/6:2) {};
                \draw[->, >=stealth] (B\i) -- (A);
            }
            \node (B) at (90+360:2) {};
            \draw[dashed, ->] (B) -- (A);
        \end{tikzpicture}
    \end{subfigure}
    \hfill
    \begin{subfigure}[b]{0.3\textwidth}
        \centering
        \begin{tikzpicture}[->, >=Stealth, scale =.7]
            \node (A) {V};
            \node[circle, draw = none] (Blanc) at (0, -2) {};
            \foreach \i in {5,7} {
                \node[circle, draw = none] (B\i) at (90+\i*360/6:2) {};
                \draw[->, >=stealth] (B\i) -- (A);
            }
            \node (B) at (90+360:2) {};
            \draw[dashed, ->] (B) -- (A);
            \draw[ ->] (A) -- (Blanc);
        \end{tikzpicture}
    \end{subfigure}
    \hfill
    \begin{subfigure}[b]{0.3\textwidth}
        \centering
        \begin{tikzpicture}[->, >=Stealth, scale =.7]
            \node (A) {V};
            \node[circle, draw = none] (Blanc) at (0, -2) {};
            \foreach \i in {2,4} {
                \node[circle, draw = none] (B\i) at (90+\i*360/6:2) {};
                \draw[->, >=stealth] (A) -- (B\i) ;
            }
            \node (B) at (90+360:2) {};
        \end{tikzpicture}
    \end{subfigure}
    \caption{We represent the three forms that arrows can take around a vertex with multiple arrow in a structure of interest. Some vertices have incoming arrows (left), some have incoming arrows and a single outgoing arrow (middle), and some have exactly two outgoing arrows and no incoming arrows (right).}
\label{fig:arrows_in_structure_of_interest}
\end{figure}

During the proofs, we are often led to construct structures of interest in which a root lies outside the target set. However, the graph containing such a structure also includes a directed path from this problematic root (outside the target set) to a vertex within the target set. Lemma~\ref{lemma:remove_root} shows how to exploit this directed path to construct a new structure of interest where the problematic root has been removed. Figure~\ref{fig:example1} depicts this idea.

\begin{figure}[ht]
    \centering
    \begin{subfigure}{0.45\textwidth}
    \centering
    \begin{tikzpicture}[->, >=Stealth]

    \node[text=black, font=\bfseries] (A) at (0,0) {$A$};
    \node[text=black, font=\bfseries] (B) at (2,0) {$B$};
    \node[text=black, font=\bfseries] (C) at (1,-1) {$C$};
    \node[text=black] (R1) at (0,-2) {$R_1$};
    \node[text=black, font=\bfseries] (R2) at (2,-2) {$R_2$};
    
    \draw[<->, dashed, , line width=1.1pt] (A) to[bend left=25] (B);
    \draw[->, line width=1.1pt] (A) -- (C);
    \draw[->, draw=gray] (B) -- (C);
    \draw[->, line width=1.1pt] (C) -- (R2);
    \draw[->, draw=gray] (C) -- (R1);
    
    \node[draw=black, dotted, fit=(R1)(R2), ellipse, inner sep=5pt, scale=.9] {};

    \end{tikzpicture}
        \caption{}
        \label{fig:example1:a}
    \end{subfigure}
    \hfill
    \begin{subfigure}{0.45\textwidth}
    \centering
    \begin{tikzpicture}[->, >=Stealth]
    
    \node[text=black, font=\bfseries] (A) at (0,0) {$A$};
    \node[text=black, font=\bfseries] (B) at (2,0) {$B$};
    \node[text=black, font=\bfseries] (C) at (1,-1) {$C$};
    \node[text=black] (R1) at (0,-2) {$R_1$};
    \node[text=black, font=\bfseries] (R2) at (2,-2) {$R_2$};
    
    \draw[<->, dashed, , line width=1.1pt] (A) to[bend left=25] (B);
    \draw[->, line width=1.1pt] (A) -- (C);
    \draw[->, line width=1.1pt] (B) -- (C);
    \draw[->, line width=1.1pt] (C) -- (R2);
    \draw[->, draw=gray] (C) -- (R1);
    
    \node[draw=black, dotted, fit=(R1)(R2), ellipse, inner sep=5pt, scale=.9] {};
    
    \end{tikzpicture}
    \caption{}
    \label{fig:example1:b}
    \end{subfigure}
    \caption{Figure~\ref{fig:example1:a} depicts a graph containing a structure of interest shown in bold black. Let us assume that the target set of roots is $\{R_1, R_2\}$. In this graph, $B$ is a root outside the target set. However, the graph contains a directed path $\langle B, C, R_1 \rangle$ from $B$ to a vertex in the target set. Figure~\ref{fig:example1:b} illustrates how this path can be used to construct a structure of interest in which $B$ is no longer a root, without introducing a new root outside the target set.}
    \label{fig:example1}
\end{figure}
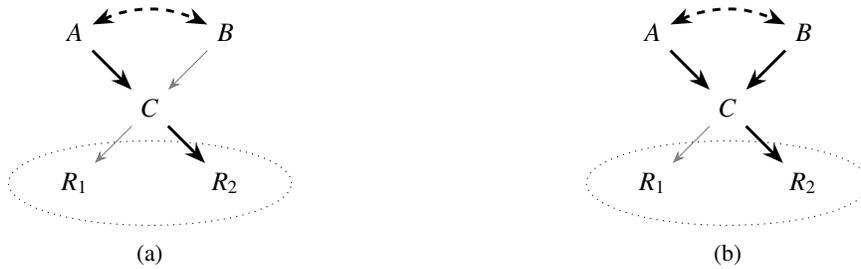

\begin{lemma}[Add a Path to Remove a Problematic Root]
\label{lemma:remove_root}
    Let $\G$ be an ADMG. Let $\X$, $\Y$ and $\Z$ be pairwise distinct subsets of nodes of $\G$. Let $\sigma$ be a structure of interest such that:
    \begin{itemize}
        \item $\sigma \subseteq \G$
        \item \(\X \cap \sigma \neq \emptyset\) and \(\Y \cap \sigma \neq \emptyset\).
        \item \((\sigma \setminus \Root(\sigma)) \cap \Z = \emptyset\). 
    \end{itemize}
    If there exists $R \in \left(\Root(\sigma) \setminus \Z\right) \cap \Anc(\Z, \G)$ then $\G$ contains a structure of interest $\sigma'$ such that:
    \begin{itemize}
        \item[(a)] $\sigma' \subseteq \G$
        \item[(b)] \(\X \cap \sigma' \neq \emptyset\) and \(\Y \cap \sigma' \neq \emptyset\).
        \item[(c)] \((\sigma' \setminus \Root(\sigma')) \cap \Z = \emptyset\). 
        \item[(d)] $\Root(\sigma') \setminus \Z \subseteq \left( \Root(\sigma) \setminus \Z \right) \setminus \{R\}$.
    \end{itemize}

\end{lemma}
\begin{proof}[Sketch of proof]
    $R \in \Anc(\Z, \G)$, thus $\G$ contains a directed path $\pi$ from $R$ to $\Z$. We add this path to $\sigma$ to remove $R$ from the root set. The end of the path is in $\Z$, thus we do not add a problematic root.
\end{proof}

\begin{proof}
    $R \in \Anc(\Z, \G)$, thus $\G$ contains a directed path $\pi$ from $R$ to $\Z$. Without loss of generality, we assume that $\pi$ meets $\Z$ only at its last vertex. We construct $\sigma'$ with the following procedure:
    \begin{itemize}
        \item If $\sigma \cap \pi \setminus \{R\} = \emptyset$, then $\sigma \cup \pi$ is a structure of interest. In this case, we set $\sigma' \gets \sigma \cup \pi$. We have the following properties:
        \begin{itemize}
            \item[(a)] $\sigma' \subseteq \G$ because $\sigma \subseteq \G$ and $\pi \subseteq \G$.
            
            \item[(b)] \(\X \cap \sigma \subseteq \X \cap \sigma'\) and \(\Y \cap \sigma \subseteq \Y \cap \sigma'\) thus \(\X \cap \sigma' \neq \emptyset\) and \(\Y \cap \sigma' \neq \emptyset\).
            
            \item[(c)] By construction, $\Root(\sigma') = \Root(\sigma) \cup \Root(\pi) \setminus \{R\}$. Thus, $\sigma' \setminus \Root(\sigma') = \{R\} \cup (\sigma \cup \pi) \setminus(\Root(\sigma) \cup \Root(\pi)) \subseteq \{R\} \cup (\sigma \setminus \Root(\sigma)) \cup (\pi \setminus \Root(\pi))$. Since $R \notin \Z$, \((\sigma \setminus \Root(\sigma)) \cap \Z = \emptyset\), and \((\pi \setminus \Root(\pi)) \cap \Z = \emptyset\) (because  $\pi$ meets $\Z$ only at its last vertex), we can conclude that \((\sigma' \setminus \Root(\sigma')) \cap \Z = \emptyset\).
            
            \item[(d)]  $\Root(\sigma') = \Root(\sigma) \cup \Root(\pi) \setminus \{R\}$. Since $\Root(\pi) \subseteq \Z$, we conclude that $\Root(\sigma') \setminus \Z \subseteq \left( \Root(\sigma) \setminus \Z \right) \setminus \{R\}$.
        \end{itemize}
        
        \item Otherwise, $\sigma \cap \pi \setminus \{R\} \neq \emptyset$. Let $W$ be the first encounter of $\pi$ and $\sigma \setminus \{R\}$ and let $\pi'$ be the subpath of $\pi$ from $R$ to $W$.\footnote{$\pi$ and $\pi'$ may be equal.} In this case, we set $\sigma' \gets \sigma \cup \pi'$. We have (a), (b), but also
        \begin{itemize}
                
            \item[(c)] By construction, $\Root(\sigma') = \Root(\sigma) \setminus \{R\}$. Thus, $\sigma' \setminus \Root(\sigma') = \{R\} \cup (\sigma \setminus \Root(\sigma)) \cup (\pi'\setminus \Root(\sigma))$. Since $\pi$ meets $\Z$ only at its last vertex, we know that $\pi' \cap \Z \subseteq \{W\}$.\footnote{$\pi' \cap \Z$ is empty if $\pi' \neq \pi$, otherwise, it is equal to \{W\}.} Thus, $(\pi'\setminus \Root(\sigma)) \cap \Z = (\pi'  \cap \Z) \setminus \Root(\sigma) \subseteq \{W\} \setminus \Root(\sigma) \subseteq \sigma \setminus \Root(\sigma)$. Since $R \notin \Z$ and \((\sigma \setminus \Root(\sigma)) \cap \Z = \emptyset\), we can conclude that \((\sigma' \setminus \Root(\sigma')) \cap \Z = \emptyset\).
    
            \item[(d)]  $\Root(\sigma') = \Root(\sigma) \setminus \{R\}$. Therefore, $\Root(\sigma') \setminus \Z \subseteq \left( \Root(\sigma) \setminus \Z \right) \setminus \{R\}$.
        \end{itemize} 
        However, $\sigma'$ is not necessarily a structure of interest. By construction, $W$ is the only node of $\sigma \cup \pi'$ which does not necessarily satisfy the conditions of Definition~\ref{def:structure_of_interest}. Since $\sigma$ is a structure of interest, the arrows around $W$ in $\sigma$ are necessarily in one of the three cases described by Figure~\ref{fig:arrows_in_structure_of_interest}. The right hand case is the only case where adding an incoming arrow prevents $\sigma'$ from being a structure of interest. Thus, if $\sigma'$ is not a structure of interest, it means that $W$ has one incoming arrow and two outgoing arrows in $\sigma'$. Let $A$ and $B$ be the two children of $W$ in $\sigma'$. Moreover, since $\sigma \cap \pi \neq \emptyset$, we know that $\sigma'$ has a single connected component. Let $\pi^\star$ be a path between $X \in \X$ and $Y \in \Y$ in $\sigma'$. Let us show that we can assume that $\pi^\star$ does not use $A \leftarrow W \rightarrow B$:
        \begin{itemize}
            \item If  $\pi^\star$ uses $A \leftarrow W \rightarrow B$. Without loss of generality, we assume that $A$ is before $B$ in $\pi^\star$. Let $\pi^{XR}$ be a path from $X$ to $R$ in $\sigma$. We distinguish two cases:
            \begin{itemize}
                \item If $\pi^{XR}$ does not encounter $W$ (cf Figure~\ref{fig:example2:a}). We consider $\theta \coloneqq \pi^{XR} \cup \pi' \cup \pi^\star_{[W,Y]}$. $\theta$ is a subgraph of $\sigma'$ in which $\X$ and $\Y$ are connected.\footnote{$\theta$ is not necessarily a path.} By construction, $\theta$ does not contain $A \leftarrow W$. Thus $\theta \subseteq \sigma'$ contains a path between $\X$ and $\Y$ that does not use $A \leftarrow W \rightarrow B$.

                \item Otherwise, $\pi^{XR}$ encounters $W$. Since $\pi^{XR} \subseteq \sigma$, $\pi^{XR}$ uses $A \leftarrow W \rightarrow B$. We distinguish two cases:
                \begin{itemize}
                    \item If $A$ is before $B$ in $\pi^{XR}$ (cf Figure~\ref{fig:example2:b}), we consider $\theta \coloneqq \pi^\star_{[XW]} \cup \pi' \cup \pi^{XR}_{[R,B]} \cup \pi^\star_{[B,Y]}$. $\theta$ is a subgraph of $\sigma'$ in which $\X$ and $\Y$ are connected. By construction, $\theta$ does not contain $W \rightarrow B$. Thus $\theta \subseteq \sigma'$ contains a path between $\X$ and $\Y$ that does not use $A \leftarrow W \rightarrow B$.
                    
                    \item Otherwise, $A$ is after $B$ in $\pi^{XR}$ (cf Figure~\ref{fig:example2:c}), we consider $\theta \coloneqq  \pi^{XR}_{[X,B]} \cup \pi^\star_{[B,Y]}$. $\theta$ is a subgraph of $\sigma'$ in which $\X$ and $\Y$ are connected. By construction, $\theta$ does not contain $A \leftarrow W \rightarrow B$. Thus $\theta \subseteq \sigma'$ contains a path between $\X$ and $\Y$ that does not use $A \leftarrow W \rightarrow B$.
                \end{itemize}
            \end{itemize}
        \end{itemize}
        Thus, without loss of generality, we assume that $\pi^\star$ does not use $A \leftarrow W \rightarrow B$. We remove from $\sigma'$ one outgoing arrow from $W$ that is not used by $\pi^\star$. By doing so, $W$ satisfies the conditions of Definition \ref{def:structure_of_interest}. Thus, all the vertices in $\sigma'$ now satisfy the conditions of Definition \ref{def:structure_of_interest}, and satisfies (a) and (b), and
        \begin{itemize}
                
            \item[(c)] \((\sigma' \setminus \Root(\sigma')) \cap \Z = \emptyset\) because removing the arrow does not change the vertices nor the roots of $\sigma'$.
    
            \item[(d)] $\Root(\sigma') \setminus \Z \subseteq \left( \Root(\sigma) \setminus \Z \right) \setminus \{R\}$ because removing the arrow does not change the roots of $\sigma'$.
        \end{itemize} 
        
        However, $\sigma'$ does not contain necessarily a single connected component. Thus, we only keep the connected component of $X$, which contains $Y$ via $\pi^\star$. By doing so, $\sigma'$ is now a structure of interest and we have:
        \begin{itemize}
        \item[(c)] \((\sigma' \setminus \Root(\sigma')) \cap \Z = \emptyset\) because we consider a subgraph.
        \item[(d)] $\Root(\sigma') \setminus \Z \subseteq \left( \Root(\sigma) \setminus \Z \right) \setminus \{R\}$ because we consider a subgraph.
    \end{itemize}
    \end{itemize}
\end{proof}

\begin{figure}[ht]
    \centering

    \begin{subfigure}{0.5\textwidth}
        \centering
        \begin{tikzpicture}[>=Stealth, scale=1.1]
                \node (R) at (0,0) {$R$};
                \node (W) at (0,-1) {$W$};
                \node (A) at (-1,-2) {$A$};
                \node (B) at (1,-2) {$B$};
                \node (X) at (-2,0) {$X$};
                \node (Y) at (2,0) {$Y$};
                
                \draw[->] (W) -- (A);
                \draw[->] (W) -- (B);
                \draw[decorate, decoration={snake, amplitude=0.6mm, segment length=2mm}] (R) -- node[midway, right, font=\scriptsize] {$\pi'$} (W);
                \draw[decorate, decoration={snake, amplitude=0.6mm, segment length=4mm}] (X) -- (A);
                \draw[decorate, decoration={snake, amplitude=0.6mm, segment length=4mm}] (Y) -- (B);
                \draw[decorate, decoration={snake, amplitude=0.6mm, segment length=4mm}] (X) -- (R);
                
                \draw[dotted, transform canvas={yshift=3mm}]
                plot[smooth, tension=.5]
                coordinates {(X) (R)}
                node[midway, above,xshift=-1cm, font=\scriptsize] {$\pi^{XR}$};
        
                \draw[dotted, transform canvas={yshift=-5mm}]
                    plot[smooth, tension=.5]
                    coordinates {(X) (A) (W) (B) (Y)}
                    node[midway, below, yshift=-11mm, font=\scriptsize] {$\pi^\star$};
        \end{tikzpicture}
        \caption{}
        \label{fig:example2:a}
    \end{subfigure}

    \begin{subfigure}{0.45\textwidth}
        \centering
        \begin{tikzpicture}[>=Stealth, scale=1.1]
                \node (R) at (0,0) {$R$};
                \node (W) at (0,-1) {$W$};
                \node (A) at (-1,-2) {$A$};
                \node (B) at (1,-2) {$B$};
                \node (X) at (-2,0) {$X$};
                \node (Y) at (2,0) {$Y$};
                
                \draw[->] (W) -- (A);
                \draw[->] (W) -- (B);
                \draw[decorate, decoration={snake, amplitude=0.6mm, segment length=2mm}] (R) -- node[midway, left, font=\scriptsize] {$\pi'$} (W);
                \draw[decorate, decoration={snake, amplitude=0.6mm, segment length=4mm}] (X) -- (A);
                \draw[decorate, decoration={snake, amplitude=0.6mm, segment length=4mm}] (Y) -- (B);
                
                \path (X) edge[decorate, decoration={snake, amplitude=0.3mm, segment length=2mm}, bend left=35]  (A);
                
                \path (R) edge[decorate, decoration={snake, amplitude=0.3mm, segment length=2mm}, bend left=35]  (B);

                \draw[dotted, transform canvas={yshift=-5mm}]
                    plot[smooth, tension=.5]
                    coordinates {(X) (A) (W) (B) (Y)}
                    node[midway, below, yshift=-11mm, font=\scriptsize] {$\pi^\star$};

                \draw[dotted]
                    plot[smooth, tension=.7]
                    coordinates {(-1.7,0) (-1.275, -0.275) (-.95,-.75) (-.65,-1.4) (0.02,-.8) (.9,-1.6) (0.725, -0.85) (.15,-.25)}
                    node[midway, below,xshift=-8mm, yshift=-1mm, font=\scriptsize] {$\pi^{XR}$};
                
        \end{tikzpicture}
        \caption{}
        \label{fig:example2:b}
    \end{subfigure}
    \hfill
    \begin{subfigure}{0.45\textwidth}
        \centering
        \begin{tikzpicture}[>=Stealth, scale=1.1]
                \node (R) at (0,0) {$R$};
                \node (W) at (0,-1) {$W$};
                \node (A) at (-1,-2) {$A$};
                \node (B) at (1,-2) {$B$};
                \node (X) at (-2,0) {$X$};
                \node (Y) at (2,0) {$Y$};
                
                \draw[->] (W) -- (A);
                \draw[->] (W) -- (B);
                \draw[decorate, decoration={snake, amplitude=0.6mm, segment length=2mm}] (R) -- node[midway, right, font=\scriptsize] {$\pi'$} (W);
                \draw[decorate, decoration={snake, amplitude=0.6mm, segment length=4mm}] (X) -- (A);
                \draw[decorate, decoration={snake, amplitude=0.6mm, segment length=4mm}] (Y) -- (B);
                
                \path (X) edge[decorate, decoration={snake, amplitude=0.3mm, segment length=2mm}, bend left=55,looseness = 1.8]  (B);
                
                \path (A) edge[decorate, decoration={snake, amplitude=0.3mm, segment length=2mm}, bend left=25]  (R);
        
                \draw[dotted, transform canvas={yshift=-5mm}]
                    plot[smooth, tension=.5]
                    coordinates {(X) (A) (W) (B) (Y)}
                    node[midway, below, yshift=-11mm, font=\scriptsize] {$\pi^\star$};

                \draw[dotted]
                    plot[smooth, tension=.7]
                    coordinates {(-1.7,0) (.35,.3) (.9,-1.6) (0.02,-.8) (-.775,-1.5) (-0.65, -0.775) (-.15,-.25)}
                    node[midway, below,xshift=-10mm, yshift=3mm, font=\scriptsize] {$\pi^{XR}$};
                    
        \end{tikzpicture}
        \caption{}
        \label{fig:example2:c}
    \end{subfigure}

    \caption{Helping figure for the Proof of Lemma \ref{lemma:remove_root}. $\pi^\star$ and $\pi^{XR}$ are represented by the dotted paths following the arrows. A squiggly arrow represents an arbitrary path.}
    \label{fig:example2}
\end{figure}
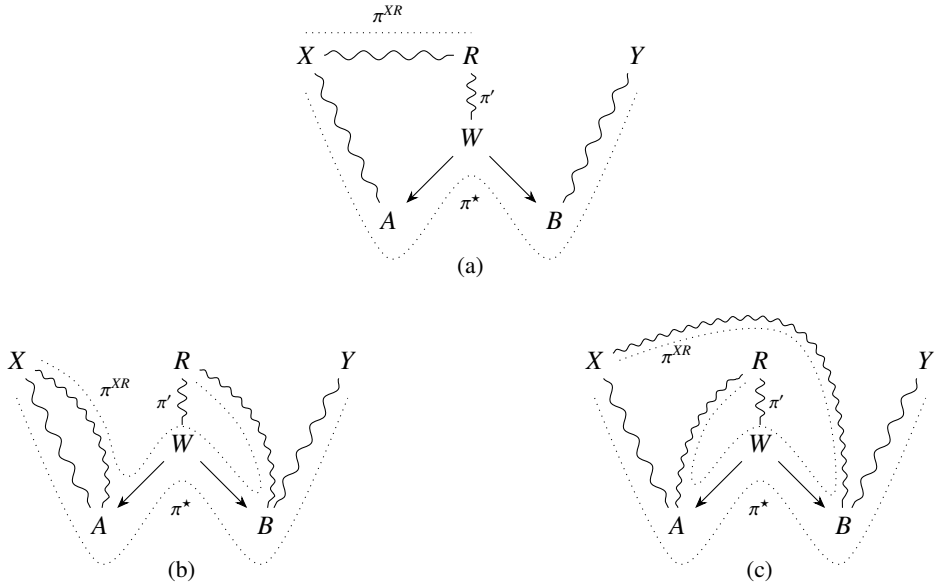

\begin{corollary}
\label{cor:add_paths}
    Let $\G$ be an ADMG. Let $\X$, $\Y$, $\Z$ and $\mathcal{R}$ be pairwise distinct subsets of nodes of $\G$. Let  $\G$ contains a structure of interest $\sigma$ such that:
    \begin{itemize}
        \item $\X \cap \sigma \neq \emptyset$ and $\Y \cap \sigma \neq \emptyset$
        \item $\Root(\sigma) \subseteq \Z \cup \X \cup \Y \cup \mathcal{R}$
        \item $(\sigma \setminus \Root(\sigma)) \cap \Z = \emptyset$
    \end{itemize}
     If for all $R \in \mathcal{R}$, $R \in \Anc(\Z, \G)$, then $\G$ contains a structure of interest that connects $\X$ and $\Y$ under $\Z$.
\end{corollary}

\begin{proof}
    We apply Lemma \ref{lemma:remove_root} iteratively for each $R \in \mathcal{R}$, we get a structure of interest $\sigma'$ such that:
    \begin{itemize}
        \item $\sigma' \subseteq \G$
        \item $\X \cap \sigma' \neq \emptyset$ and $\Y \cap \sigma' \neq \emptyset$
        \item $\Root(\sigma') \setminus \Z \subseteq (\Root(\sigma) \setminus \Z) \setminus \mathcal{R}$. Thus, $\Root(\sigma) \subseteq \Z \cup \X \cup \Y$.
        \item $(\sigma' \setminus \Root(\sigma')) \cap \Z = \emptyset$
    \end{itemize}
    Therefore, $\sigma'$ a structure of interest that connects $\X$ and $\Y$ under $\Z$.
\end{proof}

\thnewdsep*

\begin{proof}
Let us prove the two implications:
\begin{itemize}
    \item $\ref{th:new_d_sep:1} \Rightarrow \ref{th:new_d_sep:2}$: If \(\X \notind_{\G} \Y \mid \Z\), then, by definition, the ADMG \(\G\) contains a path \(\pi\) d-connecting \(\X\) and \(\Y\). Without loss of generality, we assume that $\pi$ encounters $\X$ only at its first vertex and $\Y$ only at its last vertex.  Since \(\pi\) is a d-connecting path, it is a structure of interest which satisfies the following properties:
    \begin{itemize}
        \item \(\pi \subseteq \G\).
        \item \(\X \cap \pi \neq \emptyset\) and \(\Y \cap \pi \neq \emptyset\).
        \item \((\pi \setminus \Root(\pi)) \cap \Z = \emptyset\).
    \end{itemize}
    
    However, some roots of \(\pi\) may not be in $\Z \cup \X \cup \Y$, preventing \(\pi\) from being a connecting structure of interest. Necessarily, theses roots are colliders. Define $\mathcal{R} \coloneqq \Root(\pi) \setminus \left( \Z \cup \X \cup \Y \right)$ to be the set of these colliders. Since $\pi$ is d-connecting, for all $R \in \mathcal{R}$, $R \in \Anc(\Z,\G)$. By Corollary~\ref{cor:add_paths}, $\G$ contains a structure of interest which connects \(\X\) and \(\Y\) under \(\Z\).

    \item $\ref{th:new_d_sep:2} \Rightarrow \ref{th:new_d_sep:1}$: By definition, $\sigma$ has a single connected component. Thus $\sigma$ contains a path $\pi$ from $\X$ to $\Y$. Without loss of generality, we assume that $\pi$ encounters $\X$ only at its first vertex and $\Y$ only at its last vertex. Let us first prove that without loss of generality, we can assume that all colliders on $\pi$ are ancestors of $\Z$.
    If it's not the case, since $\Root(\sigma) \subseteq \Z \cup \X \cup \Y$, all colliders that are not ancestors of $\Z$ are ancestors of $\X \cup \Y$. Without loss of generality, assume that a collider \(C\) on \(\pi\) is an ancestor of \(\X\). Thus, $\sigma$ contains a directed path $\pi^1$ from $C$ to $\X$. Let $\pi^2$ be the subpath of $\pi$ between $C$ and $\Y$. Since \(C\) belongs to both \(\pi^1\) and \(\pi^2\), these two paths intersect at \(C\), and possibly at other vertices. Let $T$ be the last vertex of $\pi^1$ that is in $\pi^1 \cap \pi^2$. Let $\pi' \coloneqq \pi^1_{[T,\X]} \cup \pi^2_{[T,\Y]}$. $\pi'$ is a path. Moreover, between $X$ and $T$,  $\pi'$ is a directed path, and, after $T$ to $\Y$, $\pi'$ is a subpath of $\pi_2$. Therefore, $\pi'$ contains at least one fewer collider than $\pi$ that is not an ancestor of $\Z$. Repeating this procedure iteratively allows us to construct a path in \(\sigma\) from \(\X\) to \(\Y\) in which all colliders are ancestors of \(\Z\).

    Finally, note that all forks and chains on \(\pi\) are not roots by definition, and hence do not belong to \(\Z\). Therefore, the resulting path \(\pi\) is \(d\)-connecting between \(\X\) and \(\Y\) given \(\Z\).
\end{itemize}
\end{proof}

\subsection{Proofs of the Properties of the canonical compatible Graph and the Unfolded Graph}

\subsubsection{Canonical Compatible Graph} 
\propgmmin*

We break the proof of Proposition~\ref{prop:gm_min} in Lemma~\ref{lemma:G3_compatible} and Lemma~\ref{lemma:gm_cup_g3_compatible}.

\begin{lemma}
\label{lemma:G3_compatible}
    Let $\Gc$ be a C-DAG. Then, $\Gt$, its canonical compatible graph, is compatible with $\Gc$.
\end{lemma}

\begin{proof}
By construction, all arrows in $\Gc$ are represented in $\Gt$. Moreover, all arrows that are added correspond to an arrow in $\Gc$. Therefore, we only need to check for acyclicity to prove that $\Gt$ is a compatible graph with $\Gc$. 
By contradiction, let us assume that $\Gt$ contains a cycle $\pi$. First, we know that $\pi$ is not within a single cluster. Indeed, in each cluster $V$, $\Gt_{\mid V}$ is a $V_{\#V}$-rooted tree. Thus, $\pi$ encounters at least two clusters and has an arrow between two clusters. Let $A_1 \rightarrow B_{\#B}$ be such an arrow. We prove that necessarily, $\#A^C = 1$. Indeed, if $\#A^C \geq 2$, then no arrow in $\Gt$ is pointing on $A_1$. Similarly, we can show that $\#B^C = 1$. Thus, the next arrow cannot be pointing into $B^C$, thus the next arrow is also an arrow between two clusters. By induction, we show that all the clusters encountered by $\pi$ have a cardinal of $1$. This contradicts Proposition \ref{prop:valid_gc}. Therefore, $\Gt$ is acyclic. 

Therefore, $\Gt$ is a compatible graph with $\Gc$.
\end{proof}

\begin{lemma}
\label{lemma:gm_cup_g3_compatible}
    Let $\Gc$ be a C-DAG, $\Gt$ be its corresponding canonical compatible graph, and $\Gm$ be a compatible graph. Then,  $\Gt \cup \Gm$  is a compatible graph.
\end{lemma}

\begin{proof}
Since $\Gm$ and $\Gt$ are compatible, then we only need to check for acyclicity. Let label indices according to Notation \ref{notation:base}. Let $a$ be an arrow in $\Gt$ that is not in $\Gm$. We distinguish three cases:
\begin{itemize}
    \item If $a$ is a dashed-bidirected arrow, then adding $a$ does not create a cycle.
    
    \item If $a$ is a directed arrow inside a cluster $V^C$, since $\Gm$ follows Notation \ref{notation:base}, then the indices in $V^m$ follow a topological ordering associated with $\Gm$. Therefore, $a$ can be added without creating a cycle.
    
    \item Otherwise, $a$ corresponds to an arrow between two clusters $V^m$ and $U^m$. $\Gm$ is compatible thus it also contains an arrow from $U^m$ to $V^m$. By applying Propositions \ref{prop:move_arrow_up} and \ref{prop:move_arrow_down}, we see that we can add $a$ without creating a cycle.
\end{itemize}

Therefore, $\Gt \cup \Gm$  is a compatible graph.
\end{proof}

\subsubsection{Unfolded Graph}

\propgcu*

\begin{proof}
    Let $\Gm =\left( \mathcal{V}^m, \mathcal{E}^m  \right)$ be a compatible graph. By definition, we already know that $\mathcal{V}^m = \mathcal{V}_{\text{u}}$.  $\Gm$ is a DAG, thus, in each cluster, we can permute the indices of the vertices so that a topological order of $\Gm$ agrees with the order of the indices. Let $a$ be an arrow of $\Gm$. We distinguish three cases:
    \begin{itemize}
        \item If $a$ is a dashed-bidirected-arrow, then $a$ is also in $\Gcu$ because $\Gcu$  is a super graph of $\Gt$ which contains all possible dashed-bidirected-arrows.
        
        \item If $a$ corresponds to a self-loop $\leftselfloop V^C$ in $\Gc$. Necessarily, in $\Gm$, $a = V_i \rightarrow V_j$ with $i<j$. Thus, $a$ corresponds to an arrow added during step \ref{def:gm_min:2} of Definition \ref{def:gm_min}. Therefore, $a$ is also an arrow in $\Gcu$.
        
        \item Otherwise, $a$ corresponds to an arrow $U^C \rightarrow V^C$, with $U^C \neq V^C$. We distinguish two cases:
        \begin{itemize}
            \item If $a$ is added at step \ref{def:gm_min:3} of Definition \ref{def:gm_min}, then $a$ is also an arrow in $\Gcu$.
            
            \item Otherwise, by Lemma \ref{lemma:gm_cup_g3_compatible}, we know that $\Gt \cup \Gm$ is compatible. Thus  $\Gt \cup \Gm$ is acyclic. Since $\Gt \cup a$ is a subgraph of $\Gt \cup \Gm$, we can conclude that $a$ does not create a cycle in $\Gt$. Thus, $a \in \E_{\text{eligible}}$. Therefore $a$ is  an arrow in $\Gcu$
        \end{itemize}
    \end{itemize}
    Hence, $\mathcal{E}^m \subseteq \mathcal{E}_{\text{u}}$. Therefore, $\Gm$ is a subgraph of $G_{\text{u}}$.
    
\end{proof}

\subsubsection{Proof of Proposition \ref{prop:gcugcan}}

\propgcuetgcan*

\begin{proof}
    By Proposition \ref{prop:gm_min}, $\Gt$ is compatible with $\Gc$. Therefore, $\Gt \cup \sigma^m$ contains all necessary arrows to be compatible with $\Gc$. Moreover, by definition of $\Gcu$, all arrow $V_v \rightarrow W_w$ in $\sigma^m \subseteq \Gcu$ correspond to an arrow $V \rightarrow W$ in $\Gc$. Therefore, $\Gt \cup \sigma^m$ does not contain any arrow preventing it from being compatible with $\Gcu$. $\Gt \cup \sigma^m$ is acyclic, thus it is an ADMG. Therefore, $\Gt \cup \sigma^m$ is compatible with $\Gc$.
\end{proof}

\subsection{Proofs of the Calculus}

First of all let us recall the rules of Pearl's calculus for cluster queries.

\begin{theorem}[Do-Calculus Rules for Cluster Queries \citep{pearl_causality_2009}]
Let \(\Gc\) be a C-DAG and \(\X^C,\Y^C,\Z^C,\mathcal{W}^C\) be pairwise distinct subsets of nodes. 
Let $\Gm$ be a compatible graph. The following rules hold:
\begin{enumerate}
    \item \textbf{Insertion/deletion of observations:}
    
    \(
    P\bigl(\mathbf{y^m} \mid \Do(\mathbf{w^m}),\mathbf{x^m}, \mathbf{z^m}) = P\bigl(\mathbf{y^m} \mid \Do(\mathbf{w^m}),\mathbf{z^m})
    \)
    \quad if $\Y^m \ind_{\G_{\overline{\W^m}}} \X^m \mid \W^m, \Z^m$
    
    \item \textbf{Action/observation exchange:}
    
    \(
    P\bigl(\mathbf{y^m} \mid \Do(\mathbf{w^m}), \Do(\mathbf{x^m}), \mathbf{z^m}) = P\bigl(\mathbf{y^m} \mid \Do(\mathbf{w^m}),\mathbf{x^m}, \mathbf{z^m})
    \)
    \quad if $\Y^m \ind_{\G_{\overline{\W^m}, \underline{\X^m}}} \X^m \mid \W^m, \Z^m$

    \item \textbf{Insertion/deletion of actions:}
    
    \(
    P\bigl(\mathbf{y^m} \mid \Do(\mathbf{w^m}), \Do(\mathbf{x^m}), \mathbf{z^m}) = P\bigl(\mathbf{y^m} \mid \Do(\mathbf{w^m}),\mathbf{z^m})
    \)
    \quad if  $\Y^m \ind_{\G_{\overline{\W^m}, \overline{\X^m(\Z^m)}}} \X^m \mid \W^m, \Z^m$
    
    where $\X^m(\Z^m) \coloneqq \X^m \setminus \Anc(\Z^m, \G_{\overline{\W^m}})$.
\end{enumerate}
\end{theorem}

The first two rules hinge on graphical conditions expressed as d-separations under cluster-level mutilations. We begin by rigorously characterizing this form of dependency; Theorem \ref{th:cluster_sep_cluster_mutilation} provides precisely this characterization.

\subsubsection{Cluster D-separation with Cluster Mutilations}

\begin{theorem}
\label{th:cluster_sep_cluster_mutilation}
    Let $\Gc = \left( \mathcal{V}^C, \mathcal{E}^C  \right)$ be a C-DAG.  Let $\Gt$ be its corresponding canonical compatible graph. Let $\Gcu$ be the corresponding unfolded graph. Let $\X^C,\Y^C$ and $\Z^C$ be pairwise distinct subsets of nodes of $\Gc$. Let $\mathcal{A}$ and $\mathcal{B}$ be subsets of nodes of $\Gc$. Then the following properties are equivalent:
    
    \begin{enumerate}
        \item $ \exists~ \Gm \in \C\left(\Gc\right) ~~ \X^m \notind_{\Gm_{\overline{A^m} \underline{B^m}}} \Y^m \mid \Z^m.$ \label{th:cluster_sep_cluster_mutilation:1}
        
        \item $\Gcu_{\overline{\mathcal{A}^m} \underline{\mathcal{B}^m}}$ contains a structure of interest $\sigma^m$ such that $\X^m \notind_{\sigma^m} \Y^m \mid \Z^m$ and $\Gt \cup \sigma^m$ is acyclic. \label{th:cluster_sep_cluster_mutilation:2}
    \end{enumerate}
\end{theorem}

\begin{proof}
Let us prove the two implications:

\begin{itemize}
    \item $\ref{th:cluster_sep_cluster_mutilation:1} \Rightarrow \ref{th:cluster_sep_cluster_mutilation:2}$: Let $\Gm$ be a compatible graph such that $\Gm_{\overline{\mathcal{A}^m} \underline{\mathcal{B}^m}}$ contains a structure of interest $\sigma^m$ which connects $\X^m$ and $\Y^m$ under $\Z^m$. By Lemma \ref{lemma:gm_subgraph_gcu}, $\sigma^m \subseteq \Gm_{\overline{\mathcal{A}^m} \underline{\mathcal{B}^m}} \subseteq \Gm \subseteq \Gcu$. Moreover, since $\sigma^m \subseteq \Gm_{\overline{\mathcal{A}^m} \underline{\mathcal{B}^m}}$, we know that $\sigma^m$ does not contain any incoming arrow in $\A^m$ and no outgoing arrow from $\B^m$. Therefore, $\sigma^m \subseteq \Gcu_{\overline{\mathcal{A}^m} \underline{\mathcal{B}^m}}$ and $\sigma^m$ connects $\X^m$ and $\Y^m$ under $\Z^m$.
    
    By Lemma \ref{lemma:gm_cup_g3_compatible}, $\Gt \cup \Gm $ is a compatible graph, thus acyclic. Moreover, $\Gt \cup \sigma^m \subseteq \Gt \cup \Gm $. Therefore,  $\Gt \cup \sigma^m$ is acyclic.
    
    \item $\ref{th:cluster_sep_cluster_mutilation:2} \Rightarrow \ref{th:cluster_sep_cluster_mutilation:1}$: $\Gt \cup \sigma^m$ is acyclic. By Lemma \ref{lemma:G3_compatible}, $\Gt$ is compatible. Moreover, all arrows from $\sigma^m$ come from $\Gcu$. Therefore,  $\Gt \cup \sigma^m$ is a compatible graph. 
    
    Moreover, since $\sigma^m \subseteq \Gcu_{\overline{\mathcal{A}^m} \underline{\mathcal{B}^m}}$, we know that $\sigma^m$ does not contain any incoming arrow in $\A^m$ and no outgoing arrow from $\B^m$. Thus, $ \sigma^m \subseteq (\Gt \cup \sigma^m)_{\overline{\mathcal{A}^m} \underline{\mathcal{B}^m}}$. Since $\sigma^m$ connects $\X^m$ and $\Y^m$ under $\Z^m$, we can conclude that $\X^m \notind_{(\Gt \cup \sigma^m)_{\overline{\mathcal{A}^m} \underline{\mathcal{B}^m}}} \Y^C \mid \Z^m$. 
    
\end{itemize}
\end{proof}

Figure~\ref{fig:th:cluster_sep_cluster_mutilation} illustrates Theorem~\ref{th:cluster_sep_cluster_mutilation}.

\begin{figure}[ht]
    \centering
    \begin{subfigure}[b]{0.4\textwidth}
        \centering
        \begin{tikzpicture}[->, >=stealth, xscale = 1.5, yscale=1.5]
            \node (A) at (0, 0) {${}^1A$};
            \node (X) at (0, -1) {${}^2X$};
            \node (Y) at (0, 1) {${}^2Y$};
            \node (B) at (1,0) {${}^3B$};
            \node (Z) at (2,0) {${}^1Z$};

            \draw (Y) edge (A);
            \draw (X) edge (A);
            \draw (A) edge (B);
            \draw (B) edge (Z);
            \path (Z) edge[bend right= 35] (A);
            \draw[<->, dashed] (B) to[bend right=35] (Z);
        \end{tikzpicture}
        \caption{$\Gc$}
        \label{fig:th:cluster_sep_cluster_mutilation:1}
    \end{subfigure}
    \hspace{7pt}
    \begin{subfigure}[b]{0.4\textwidth}
        \centering
        \begin{tikzpicture}[->, >=stealth, xscale = 1.5, yscale=1.5]
            \node (A1) at (0, 0) {$A_1$};
            \node (X1) at (-.75, -1) {$X_1$};
            \node (X2) at (.25, -1) {$X_2$};
            \node (Y1) at (-.75, 1) {$Y_1$};
            \node (Y2) at (.25, 1) {$Y_2$};
            \node (B1) at (1,1) {$B_1$};
            \node (B2) at (1,0) {$B_2$};
            \node (B3) at (1,-1) {$B_3$};
            \node (Z1) at (2,0) {$Z_1$};

            \draw (Y1) edge (A1);
            \draw (X1) edge (A1);
            \draw (A1) edge (B3);
            \draw (B1) edge (Z1);
            \path (Z1) edge[bend right= 35] (A1);
            \draw[<->, dashed] (B1) to[bend left=35] (Z1);
            \draw[<->, dashed] (B2) to[bend right=30] (Z1);
            \draw[<->, dashed] (B3) to[bend right=35] (Z1);

            \draw[tochoose] (Y2) edge (A1);
            \draw[tochoose] (X2) edge (A1);
            \draw[tochoose] (A1) edge (B2);
            \draw[tochoose] (B2) edge (Z1);
        \end{tikzpicture}
        \caption{$\Gcu$}
        \label{fig:th:cluster_sep_cluster_mutilation:2}
    \end{subfigure}
    \caption{(a)  a C-DAG~$\Gc$. (b) its corresponding unfolded graphs. The plain and dashed arrows represent~$\Gt$, while the dotted arrows denote the "eligible" edges. According to Theorem~\ref{th:cluster_sep_cluster_mutilation}, we have $\forall~ \Gm \in \C\left(\Gc\right) ~ \X^m \ind_{\Gm} \Y^m \mid \Z^m$, as all structures of interest connecting $X^m$ and $Y^m$ given $Z^m$ in~$\Gcu$ include a directed path from $A_1$ to $Z_1$. Since~$\Gt$ already contains the edge $Z_1 \rightarrow A_1$, such paths would necessarily form a cycle.
    }
    \label{fig:th:cluster_sep_cluster_mutilation}
\end{figure}

\subsubsection{Proofs of the Three Rules of the Calculus}

As soon as Theorem~\ref{th:cluster_sep_cluster_mutilation} has been established, Rules~1 and~2 of the calculus follow almost immediately. In contrast, the third rule of Pearl’s do-calculus requires verifying the d-separation condition  \(
\Y^m \ind_{\Gm_{\overline{\W^m}, \overline{\X^m (\Z^m)}}} X^m \mid \W^m, \Z^m,
\) in all compatible graph $\Gm$, where \(\X^m(\Z^m) = X^m \setminus \Anc(\Z^m,\Gm_{\overline{\W^m}})\). Since \(\X^m(\Z^m)\) is not, in general, a union of clusters, the associated mutilation depends on the particular graph $\Gm$. As a result, this rule does not fall under the scope of Theorem \ref{th:cluster_sep_cluster_mutilation}. 

Nonetheless, if Rule~3 does not hold in some compatible graph $\Gm$, then there exists a structure of interest between \(\Y^m\) and \(X^m\) in \(\Gm_{\overline{\W^m}, \overline{\X^m (\Z^m)}}\). If this structure includes a root \(X_x \in \X^m\), then \(X_x\) must be an ancestor of some \(Z_z \in \Z^m\) in the mutilated graph. In such a case, we can augment the structure of interest by explicitly adding the directed path from \(X_x\) to \(\Z^m\), resulting in a new structure whose roots lie outside \(\X^m\). 

\thcalculus*

\begin{proof}
    The first two rules are proven by Theorem \ref{th:cluster_sep_cluster_mutilation}. The third one is proved by the following reasoning.
    
    We will show that if the third rule applies, then in all compatible graph $\Gm$, the third rule of Pearl's calculus applies. More precisely, we prove the contrapositive. Let $\Gm$ be a compatible graph in which the third rule does not apply. Then, \(\Gm_{\overline{\W^m}, \overline{\X^m (\Z^m)}}\) contains a structure of interest $\sigma^m$ that connects $\Y^m$ and $\X^m$ under the conditioning set $\W^m \cup \Z^m$, where \(\X^m(\Z^m) = X^m \setminus \Anc(\Z^m,\Gm_{\overline{\W^m}})\). By definition, we already know that $\sigma^m$ follows the following properties:
\begin{itemize}
    \item $\sigma^m \subseteq \Gm_{\overline{\W^m}, \overline{\X^m (\Z^m)}} \subseteq  \Gm_{\overline{\W^m}}$.
    \item $\X^m \cap \sigma^m \neq \emptyset$ and $\Y^m \cap \sigma^m \neq \emptyset$.
    \item \(\Root(\sigma^m) \subseteq (\W^m \cup \Z^m) \cup \Y^m \cup \X^m\)
    \item $(\sigma^m \setminus \Root(\sigma^m)) \cap (\W^m \cup \Z^m) = \emptyset$
\end{itemize}

Let us remark that $\Root(\sigma^m) \cap \X^m \subseteq  \X^m \setminus \X^m(\Z^m) = \Anc(\Z^m,\Gm_{\overline{\W^m}})$. Indeed, let $X_x$ be an element of $\Root(\sigma^m) \cap \X^m$. Since $X_x$ is a root and $\sigma^m$ has a single connected component, $X_x$ must have an incoming edge within $\sigma^m$. Thus, $X_x$ must have an incoming edge within \(\Gm_{\overline{\W^m}, \overline{\X^m (\Z^m)}}\).  Thus $X_x \notin \X^m(\Z^m)$. For element  of $\Root(\sigma^m) \cap \X^m$, we iteratively update $\sigma^m$ using Lemma~\ref{lemma:remove_root}. At the end of this process, we obtain a structure of interest ${\sigma^m}'$ which satisfies the following identities:
\begin{itemize}
    \item ${\sigma^m}' \subseteq  \Gm_{\overline{\W^m}}$.
    \item $\X^m \cap {\sigma^m}' \neq \emptyset$ and $\Y^m \cap {\sigma^m}' \neq \emptyset$.
    \item $({\sigma^m}'\setminus \Root({\sigma^m}') \cap (\W^m \cup \Z^m) = \emptyset$
    \item $\Root({\sigma^m}') \setminus (\W^m \cup \Z^m) 
        \subseteq (\Root({\sigma^m}) \setminus (\W^m \cup \Z^m)) \setminus ( \Root(\sigma^m) \cap \X^m)
        \subseteq \Y^m$
\end{itemize} 

Thus, \(\Root({\sigma^m}') \subseteq (\W^m \cup \Z^m) \cup \Y^m\) and ${\sigma^m}'$ is a structure of interest which connects $\X^m$ and $\Y^m$ under $\W^m \cup \Z^m$. Since \(\Gm_{\overline{\W^m}} \subseteq {\Gcu}_{\overline{\W^m}}\), it follows that $\Gcu$ contains ${\sigma^m}'$. Moreover, since \(\Gt \cup {\sigma^m}' \subseteq \Gt \cup \Gm\), it follows that $\Gt \cup {\sigma^m}'$ is acyclic.

Therefore, \({\Gcu}_{\overline{\W^m}}\) contains a structure of interest $\sigma^m$ which $\X^m$ and $\Y^m$ under $\W^m \cup \Z^m$, with \(\Root(\sigma^m) \subseteq (\W^m \cup \Z^m) \cup \Y^m\), and such that \(\Gt \cup \sigma^m\) is acyclic.
\end{proof}

\thatomiccompleteness*

\begin{proof}
    The first two rules are proven by Theorem \ref{th:cluster_sep_cluster_mutilation}. The third one is proved by the following reasoning.
    
        If the rule does not hold, then  \( {\Gcu}_{\overline{\W^m}} \) contains a structure of interest $\sigma^m$ that connects $\X^m$ and $\Y^m$ under $\W^m \cup \Z^m$ such that $\Root(\sigma^m) \subseteq (\W^m \cup \Z^m) \cup \Y^m$ and $\Gt \cup \sigma^m $ is acyclic.  $\Gt \cup \sigma^m $ is acyclic, thus $\Gm \coloneqq  \Gt \cup \sigma^m $ is a compatible graph. Moreover, since $\sigma^m \subseteq {\Gcu}_{\overline{\W^m}} $, then $\sigma^m \subseteq \Gm_{\overline{\W^m}}$.  We will show that we can assume that $\sigma^m \subseteq \Gm_{\overline{\W^m},\overline{\X^m (\Z^m)}}$, where \(\X^m(\Z^m) = X^m \setminus \Anc(\Z^m,\Gm_{\overline{\W^m}})\):
    
    By definition, $\X^m \cap \sigma^m \neq \emptyset$. Let $X_x$ be an element of $\X^m \cap \sigma^m$. We distinguish the cases:
    \begin{itemize}
        \item If $\sigma^m$ contains an outgoing arrow from $X_x$ i.e. $X_x \rightarrow \subseteq \sigma^m$. Then this arrow is not deleted by the mutilation $\overline{\X^m (\Z^m)}$. It exists in $\Gm_{\overline{\W^m},\overline{\X^m (\Z^m)}}$.
        
        \item If $\sigma^m$ contains an incoming arrow to $X_x$ i.e. $\rightarrow X_x \subseteq \sigma^m$. Since  $\Root(\sigma^m) \subseteq (\W^m \cup \Z^m) \cup \Y^m$, we know that $X_x$ is not a root. Let $R_r$ be a root corresponding to $X_x$ i.e. an element of $\Desc(X_x, \sigma^m) \cap \Root(\sigma^m)$. Since $\sigma^m \subseteq \Gm_{\overline{\W^m}}$, we know that $R_r \notin \W^m$. Thus, $R_r \in \Z^m \cup \Y^m$. We distinguish two cases:
        \begin{itemize}
            \item If $R_r \in \Z^m$. Then $X_x \notin \X^m (\Z^m)$. Therefore, $\rightarrow X_x$ is not deleted by the mutilation $\overline{\X^m (\Z^m)}$ and it exists in $\Gm_{\overline{\W^m},\overline{\X^m (\Z^m)}}$.
            
            \item Otherwise, $R_r \in \Y^m$. Thus, $\sigma^m$ contains a proper causal path from $\X^m$ to $\Y^m$. We update $\sigma^m$ to be this path. Now $\sigma^m$ have no incoming arrows on $\X^m$, thus it exists in  $\Gm_{\overline{\W^m},\overline{\X^m (\Z^m)}}$.
        \end{itemize}
    \end{itemize}
    
    Therefore, we can assume that $\sigma^m \subseteq \Gm_{\overline{\W^m},\overline{\X^m (\Z^m)}}$. Therefore, by Theorem \ref{th:new_d_sep},  \(\Y^m \notind_{\Gm_{\overline{\W^m}, \overline{\X^m (\Z^m)}}} X^m \mid \W^m, \Z^m,\) i.e. the third rule of Pearl's do-calculus does not hold in $\Gm$.
\end{proof}

\subsection{Proof of Theorem \ref{th:infinity_leq_three}}
Theorem \ref{th:infinity_leq_three} presents three equivalences, one for each rule of do-calculus. To streamline the proofs, we introduce Corollary~\ref{cor:magic}, which restates Theorem \ref{th:calculus} in a form better suited to treating all three rules in a uniform manner.

\begin{corollary}
\label{cor:magic}
   Let $\Gc$ be a C-DAG. Let $\W^C, \X^C, \Y^C$ and $\Z^C$ be pairwise disjoint subsets of nodes. Let $\W^C$, $\X^C$, $\Y^C$ and $\Z^C$ be pairwise disjoint subsets of nodes. For \(i\in\{1,2,3\}\), let \(R_i(\W, \X,\Y,\Z)\) be the \(i\)\textsuperscript{th} rule of Pearl's Calculus applied to \((\W,\X,\Y,\Z)\), and say it "does not holds in \(\G\)” whenever its associated d-separation condition in the associated mutilated graph is not satisfied. The following propositions are equivalent:
      \begin{itemize}
        \item There exists $\G^m \in \C(\Gc)$ in which $R_i(\W^m,\X^m,\Y^m,\Z^m)$ does not hold.
        
        \item $\Gcu_{\overline{\W^m},\underline{\M_i^m}}$ contains a structure of interest $\sigma^m$ such that $\X^m \notind_{\sigma^m} \Y^m \mid \W^m,\Z^m$ and $\Gt \cup \sigma^m$ is acyclic and $\Root(\sigma^m) \subseteq \R_i^m$.
    \end{itemize}
 where $\M_i^m= \begin{cases}
        \X^m &\text{if } i=2,\\
        \emptyset &\text{otherwise.}\\
    \end{cases}$
    and
    $\R_i^m= \begin{cases}
        (\W^m \cup \Z^m) \cup \Y^m &\text{if } i=3,\\
        (\W^m \cup \Z^m) \cup \X^m \cup \Y^m &\text{otherwise.}\\
    \end{cases}$.
\end{corollary}
\begin{proof}
    Directly follows from Theorem~\ref{th:calculus}.
\end{proof}

In order to prove Theorem \ref{th:infinity_leq_three}, we need to prove Lemma \ref{lemma:structure_of_interest_XY} on structures of interest. 

\begin{lemma}
\label{lemma:structure_of_interest_XY}
 Let $\G=(\V,\E)$ be a mixed graph. Let $\X, \Y, \Z$ be pairwise disjoint subsets of $\V$. Let $\sigma  \subseteq \G$ be a structure of interest such that $\X \notind_{\sigma} \Y \mid \Z$. Then there exists a structure of interest $\sigma' \subseteq \sigma$ such that $\X \notind_{\sigma'} \Y \mid \Z$ and such that $\# \sigma' \cap \X = 1$ and $\# \sigma' \cap \Y = 1$.
\end{lemma}

\begin{proof}
    $\X \notind_{\sigma} \Y \mid \Z$, thus $\sigma$ contains a d-connecting path $\pi$ from $\X$ to $\Y$. For all collider $C$ on $\pi$, $\sigma$ contains a directed path $\pi_C$ from $C$ to $\Z$. We can assume, without loss of generality, that for all collider $C$, $\pi_C$ does not encounter $\X$. Indeed, let $C^\star$ denote the last collider on $\pi$ such that $\pi_{C^\star}$ encounters $\X$. Let $X$ be the first encounter of $\X$ and $\pi_{C^\star}$. We just need to consider the path $\pi' \coloneqq {\pi_{C^\star}}_{[C^\star, X]} \cup \pi_{[C^\star, \Y]}$. Similarly, we can assume, without loss of generality that for all collider $C$, $\pi_C$ does not encounter $\Y$. We apply the construction of Lemma \ref{lemma:remove_root}. By doing so, we obtain a structure of interest $\sigma' \subseteq \pi \cup \bigcup_{C \text{ collider on } \pi} \pi_C \subseteq \sigma$ such that $\X \notind_{\sigma'} \Y \mid \Z$. Therefore,  $\# \sigma' \cap \X = 1$ and $\# \sigma' \cap \Y = 1$.
\end{proof}

\thinfinityleqthree*

\begin{proof}
    We prove the two implications.

    \textbf{Proof of $\ref{th:infinity_leq_three:2} \Rightarrow \ref{th:infinity_leq_three:1}$: }
    
    We add as many vertices without arrow as necessary to construct a graph compatible with $\Gc$ in which $R_i(\W^m,\X^m,\Y^m,\Z^m)$ does not hold.
    
    \textbf{Proof of $\ref{th:infinity_leq_three:1} \Rightarrow \ref{th:infinity_leq_three:2}$: }
    (Figure~\ref{fig:all-subfigs} illustrates the key steps of this implication with a concrete example.)

    By Corollary~\ref{cor:magic}, $\Gcu_{\overline{\W^m},\underline{\M_i^m}}$ contains a structure of interest $\sigma^m$ such that $\X^m \notind_{\sigma^m} \Y^m \mid \W^m,\Z^m$ and $\Gt \cup \sigma^m$ is acyclic and $\Root(\sigma^m) \subseteq \R_i^m$
    where $\M_i^m= \begin{cases}
        \X^m &\text{if } i=2,\\
        \emptyset &\text{otherwise.}\\
    \end{cases}$
    and
    $\R_i^m= \begin{cases}
        (\W^m \cup \Z^m) \cup \Y^m &\text{if } i=3,\\
        (\W^m \cup \Z^m) \cup \X^m \cup \Y^m &\text{otherwise.}\\
    \end{cases}$.
    
   Let $\sigma^m$ be such a structure of interest. By Lemma \ref{lemma:structure_of_interest_XY}, we can assume that $\# \sigma^m \cap \X^m = 1$ and $\# \sigma^m \cap \Y^m = 1$. Let $V^C$ be a cluster. We split $\sigma^m \cap V^m$ in two subsets as follows:
    
    \[
        \F \coloneqq \left\{ V_v \in \sigma^m \cap V^m \mid V_v \text{ has no incoming arrows in } \sigma^m  \right\}
    \]
    
    \[
        \NF \coloneqq  \sigma^m \cap V^m \setminus \F
    \]

   \textbf{We will show that we can assume that $\# \NF \leq 1$ without loss of generality.} Consider the case where $\# \NF \geq 2$. Necessarily, since $\# \sigma^m \cap \X^m = 1$ and $\# \sigma^m \cap \Y^m = 1$, we know that $V^C$ is different from any cluster in $\X^C \cup \Y^C$. Let $V_{\max \NF}$ denote the element of $\NF$ with maximal index. We distinguish two cases:
   
    \begin{itemize}
    \item If $\NF$ contains a root of $\sigma^m$. Since $\X^m \notind_{\sigma^m} \Y^m \mid \W^m,\Z^m$, we know that $\NF \subseteq (\W^m \cup \Z^m) \cup \X^m \cup \Y^m$. Since $V^C$ is different from any cluster in $\X^C \cup \Y^C$, we can conclude that $\NF \subseteq \W^m \cup \Z^m$. Since, $\X^m \notind_{\sigma^m} \Y^m \mid \W^m,\Z^m$, we know that $(\sigma^m \setminus \Root(\sigma^m)) \cap (\W^m \cup \Z^m) = \emptyset$. Therefore, all elements in $\NF$ are roots in $\sigma^m$. By Proposition \ref{prop:move_arrow_down} and since mutilations are done at cluster level, we know that for all arrows $W_w \rightarrow V_v$ with $V_v \in \mathcal{NF}$, $\Gcu_{\overline{\W^m},\underline{\M_i^m}}$ also contains the arrow $W_w \rightarrow V_{\max \NF}$. Therefore, $\Gcu_{\overline{\W^m},\underline{\M_i^m}}$ also contains a structure of interest ${\sigma^m}'$ which is equal to $\sigma^m$ except that all arrows $W_w \rightarrow V_v$ with $V_v \in \mathcal{NF}$ are now pointing toward $V_{\max \NF}$. Therefore, $V_{\max \NF}$ is the only element of ${\sigma^m}' \cap V^m$ that has incoming arrows. Therefore, in this case, we can assume that $\# \NF = 1$.
        
    \item Otherwise, every element in $\NF$ has an outgoing arrow in $\sigma^m$. Since $\sigma^m$ is a structure of interest, we know that $\sigma^m$ contains a path $\pi^m$ from $\X^m$ to $\Y^m$. We will construct a graph ${\sigma^m}'$, satisfying the following conditions:
        \begin{itemize}
            \item $\sigma^m \subseteq {\sigma^m}' \subseteq \Gcu_{\overline{\W^m},\underline{\M_i^m}}$
            \item $\Root({\sigma^m}') = \Root(\sigma^m)$
            \item ${\sigma^m}'$ differs from $\sigma^m$ only inside $V^m$.
        \end{itemize}
        
        and such that ${\sigma^m}'$ contains a path ${\pi^m}'$ from $\X^m$ to $\Y^m$ which encounters $\NF$ at most once with a chain or a collider and that this intersection occurs a $V_{\max \NF}$.
        If $\sigma^m$ and $\pi^m$ does not satisfy these conditions, let us consider $a_1$ and $a_2$ be respectively the first and last arrows of $\pi^m$ in $\NF$. We distinguish the cases:
        \begin{itemize}
            \item[1.] If both $a_1$ and $a_2$ are incoming arrows in $\NF$, i.e $\pi^m = \cdots W_w \rightcomingarrow V^1_{v^1} \cdots V^2_{v^2} \leftcomingarrow U_u \cdots $. By Proposition \ref{prop:move_arrow_down} and since mutilations are done at cluster level, we know that  $\Gcu_{\overline{\W^m},\underline{\M_i^m}}$  contains the arrows $W_w \rightcomingarrow V_{\max \NF}$ and $V_{\max \NF} \leftcomingarrow U_u$. We consider ${\sigma^m}' \coloneqq \sigma^m \cup \{W_w \rightcomingarrow V_{\max \NF} \} \cup \{V_{\max \NF} \leftcomingarrow U_u\}$ and ${\pi^m}' \coloneqq \pi^m_{[\X^m,W_w]} \cup \{W_w \rightcomingarrow V_{\max \NF}  \leftcomingarrow U_u\} \cup \pi^m_{[U_u, \Y^m]}$. Note that ${\pi^m}'$ encounters $\NF$ only at $V_{\max \NF}$ and that $\Root({\sigma^m}') = \Root(\sigma^m)$. 

            \item[2.] If both $a_1$ and $a_2$ are outgoing arrows in $\NF$, i.e $\pi^m = \cdots W_w \leftcomingarrow V^1_{v^1} \cdots V^2_{v^2} \rightcomingarrow U_u \cdots $. By Proposition \ref{prop:move_arrow_up} and since mutilations are done at cluster level, we know that  $\Gcu_{\overline{\W^m},\underline{\M_i^m}}$ contains the arrows $W_w \leftcomingarrow V_1$ and $V_1 \rightcomingarrow U_u$. We consider ${\sigma^m}' \coloneqq \sigma^m \cup \{W_w \leftcomingarrow V_1 \} \cup \{V_1 \rightcomingarrow U_u\}$ and ${\pi^m}' \coloneqq \pi^m_{[\X^m,W_w]} \cup \{W_w \leftcomingarrow V_1  \rightcomingarrow U_u\} \cup \pi^m_{[U_u, \Y^m]}$. Note that ${\pi^m}'$ does not encounter $\NF$ and that $\Root({\sigma^m}') = \Root(\sigma^m)$. 
            
            \item[3.] Otherwise, without loss of generality, we can assume that $a_1$ and $a_2$ are pointing towards $\Y^m$, i.e. $\pi^m = \cdots W_w \rightcomingarrow V^1_{v^1} \cdots V^2_{v^2} \rightcomingarrow U_u \cdots$. Indeed, otherwise, we just need to consider the path from $\Y^m$ to $\X^m$. Note that $V^1_{v^1}$ or $V^2_{v^2}$ could be equal to $V_{\max \NF}$ but not both. Since all elements in $\NF$ have an outgoing arrow, let us consider $C_c$, the child of $V_{\max \NF}$ in $\sigma^m$. By Propositions  \ref{prop:move_arrow_up}, \ref{prop:move_arrow_down} and since mutilations are done at cluster level, we know that  $\Gcu_{\overline{\W^m},\underline{\M_i^m}}$ contains the arrows $W_w \rightcomingarrow V_{\max \NF}$ and $C_c \leftcomingarrow V^2_{v^2}$. We consider ${\sigma^m}' \coloneqq \sigma^m \cup \{W_w \rightcomingarrow V_{\max \NF}\} \cup \{C_c \leftcomingarrow V^2_{v^2}\}$ and ${\pi^m}' \coloneqq \pi^m_{[\X^m,W_w]} \cup \{W_w \rightcomingarrow V_{\max \NF} \rightarrow C_c \leftcomingarrow V^2_{v^2}\} \cup \pi^m_{[V^2_{v^2}, \Y^m]}$. Note that ${\pi^m}'$  encounters $\NF$ only once with a collider in $V_{\max \NF}$ and once with a fork in $V^2_{v^2}$ and that $\Root({\sigma^m}') = \Root(\sigma^m)$.
        \end{itemize}
        
        Note that in all cases, we have used Propositions \ref{prop:move_arrow_up} and \ref{prop:move_arrow_down}, thus ${\sigma^m}' \cup \Gt$ is acyclic. Moreover, in all cases, $\Root({\sigma^m}') = \Root(\sigma^m)$. Therefore, $\Root({\sigma^m}') \subseteq \R_i^m$. In addition, ${\sigma^m}'$ contains ${\pi^m}'$, a path from $\X^m$ to $\Y^m$ which encounters $\NF$ at most once with a chain or a collider and that this intersection occurs a $V_{\max \NF}$.
        
        However, in Cases 2 and 3, we have added outgoing arrows to some vertices different from $V_{\max \NF}$. This could prevent ${\sigma^m}'$ from being a structure of interest. We apply the following transformation to construct a structure of interest from ${\sigma^m}'$:
        \begin{enumerate}
            \item \textbf{Move all incoming arrows to $\mathbf{V_{\max \NF}}$ :} By Proposition \ref{prop:move_arrow_down} and since mutilations are done at cluster level, we know that for all arrows $W_w \rightarrow V_v$  with $V_v \in NF$, $\Gcu_{\overline{\W^m},\underline{\M_i^m}}$ also contains the arrow $W_w \rightarrow V_{\max \NF}$. Therefore, $\Gcu_{\overline{\W^m},\underline{\M_i^m}}$ also contains ${{\sigma^m}'}'$ which is equal to ${\sigma^m}'$ except that all arrows $W_w \rightarrow V_v$ with $V_v \in NF$ are now pointing toward $V_{\max \NF}$. Since we are using Proposition \ref{prop:move_arrow_down}, we know that ${{\sigma^m}'}' \cup \Gt$ is acyclic. Note that ${\pi^m}'$ still exists in ${{\sigma^m}'}'$ and that all vertices in ${{\sigma^m}'}' \cap V^m$ except $V_{\max \NF}$ have no incoming arrows in ${{\sigma^m}'}'$. Moreover, note that $\Root({{\sigma^m}'}') = \Root({\sigma^m}')$. 
            
            \item \textbf{Remove problematic outgoing arrows:} To keep the notations simple, we update ${\sigma^m}' \gets {{\sigma^m}'}'$. Some vertices, different from $V_{\max \NF}$, may have more than two outgoing arrows, preventing ${\sigma^m}'$ from being a structure of interest. We remove from ${\sigma^m}'$, all outgoing arrow from $V^m$ that is not used by ${\pi^m}'$. Since, ${\pi^m}'$ uses at most two arrows around a vertex, we know that all vertices have now at most two outgoing arrows. Since we have just removed some arrows, we know that ${\sigma^m}' \cup \Gt$ remains acyclic. Moreover, since ${\pi^m}'$ is preserved, we know that $\X^m$ and $\Y^m$ are still connected.
            
            \item \textbf{Remove the problematic vertices:} At the end of the previous steps, some vertices are not connected to the others in ${\sigma^m}'$, preventing ${\sigma^m}'$ from being a structure of interest. More precisely, these vertices are $\NF \setminus \{V_{\max \NF}\}$ except $V_1$ in case 2 and except $V^2_{v^2}$ in case 3. Indeed, at step 1, they have lost their incoming arrows and at step 2, they have lost their outgoing arrows. We remove these vertices from ${\sigma^m}'$. By doing so, ${\sigma^m}' \cup \Gt$ remains acyclic and ${\sigma^m}'$ is now a structure of interest, $\Root({\sigma^m}') = \Root(\sigma^m) \subseteq \R_i^m$.
        \end{enumerate}
        
        To summarize, we have constructed a structure of interest ${\sigma^m}'$ in $\Gcu_{\overline{\W^m},\underline{\M_i^m}}$, such that $\X^m \notind_{{\sigma^m}'} \Y^m \mid \W^m,\Z^m$, $\Gt \cup {\sigma^m}'$ is acyclic and $\Root({\sigma^m}') \subseteq \R_i^m$. Moreover, $V_{\max \NF}$ is the only element of $V^m \cap {\sigma^m}'$ with incoming arrows. Therefore, in this case, we can assume that $\# \NF = 1$.
    \end{itemize}
    
    Therefore, in all cases, we can assume that $\# \NF \leq 1$ without loss of generality.
    
    \textbf{We will now show that we can assume that $\F \subseteq \{V_1\}$ without loss of generality.} If it is not the case we apply the following transformations:
    \begin{enumerate}
        \item \textbf{Move all arrows to $V_1$:} By Proposition \ref{prop:move_arrow_up} and since mutilations are done at cluster level, for all arrow $V_v \rightarrow W_w$ in $\F \cap \sigma^m$, we know that  $\Gcu_{\overline{\W^m},\underline{\M_i^m}}$ contains the arrow $V_1 \rightarrow W_w$. Thus, we consider ${\sigma^m}'$ the subgraph of $\Gcu_{\overline{\W^m},\underline{\M_i^m}}$ obtained by moving all the arrows $V_v \rightarrow W_w$ with $V_v$ in $\F \cap \sigma^m$ to $V_1 \rightarrow W_w$. Since, $\sigma^m$ is a structure of interest, it contains a path $\pi^m$ from $\X^m$ to $\Y^m$. Note that the above transformation yields a path ${\pi^m}'$ in ${\sigma^m}'$ which connects $\X^m$ and $\Y^m$.
        
        \item \textbf{Remove problematic outgoing arrows:} In ${\sigma^m}'$, $V_1$ may have more than two outgoing arrows, preventing ${\sigma^m}'$ from being a structure of interest. Yet ${\pi^m}'$ uses at most two of these arrows. Thus, we only keep two of them without altering ${\pi^m}'$.
        
        \item \textbf{Remove the problematic vertices:} At the end of the previous steps, some vertices are not connected to the others in ${\sigma^m}'$, preventing ${\sigma^m}'$ from being a structure of interest. More precisely, these vertices are $\F \setminus \{V_1\}$. Indeed, they have lost their outgoing arrow at step 1 and had no incoming arrows are they are in $\F$.
    \end{enumerate}
    
    To summarize, we have constructed a structure of interest ${\sigma^m}'$ in $\Gcu_{\overline{\W^m},\underline{\M_i^m}}$, such that $\X^m \notind_{{\sigma^m}'} \Y^m \mid \W^m,\Z^m$, $\Gt \cup {\sigma^m}'$ is acyclic and $\Root({\sigma^m}') \subseteq \R_i^m$. Moreover, $V_1$ is the only element of $V^m \cap {\sigma^m}'$ with no incoming arrows. Therefore, we can assume that $\F \subseteq \{V_1\}$.

    \textbf{End of the proof of $\ref{th:infinity_leq_three:1} \Rightarrow \ref{th:infinity_leq_three:2}$:} 
    Without loss of generality, we assume that $\# \NF \leq 1$ and $\F \subseteq \{V_1\}$. Let $V_{\NF}$ be the only element of $\NF$ (if exists). We consider the graph $\Gm \coloneqq \Gt \cup \sigma^m$. By Corollary~\ref{cor:magic},
    $R_i(\W^m,\X^m,\Y^m,\Z^m)$ does not hold in $\Gm$ and by construction $\Gm$ is compatible with $\Gc$. For every vertex in \(V^m \setminus (\F \cup \NF)\), we apply Proposition~\ref{prop:move_arrow_up} to move all outgoing arrows to \(V_1\) and Proposition~\ref{prop:move_arrow_down} to move all incoming arrows to \(V_{\#V^C}\). This yields a graph \(\Gm'\in\C(\Gc)\) in which \(R_i(\W^m,\X^m,\Y^m,\Z^m)\) does not hold, and where no vertex outside \(\NF\cup\{V_1\}\cup\{V_{\#V^C}\}\) is incident to an arrow. 
    
    We repeat this construction for each cluster, ultimately obtaining \(\Gm^\star\), where \(R_i(\W^m,\X^m,\Y^m,\Z^m)\) still fails and every cluster has arrows on at most three vertices. Finally, by removing all vertices that are not incident to any arrow in \(\Gm^\star\), we obtain \(\Gm^\star_{\le3}\), a graph compatible with \(\Gc_{\le3}\) in which \(R_i(\W^m_{\le3},\X^m_{\le3},\Y^m_{\le3},\Z^m_{\le3})\) does not hold.

    Therefore, there exists $\G^m_{\le3} \in \C(\G^C_{\le3})$ in which $R_i(\W_{\le3}^m,\X^m_{\le3},\Y^m_{\le3},\Z^m_{\le3})$ does not hold.
\end{proof}

We come back to the example presented in Figure \ref{fig:example} to illustrate the successive steps of this proof in Figure \ref{fig:all-subfigs}. Particularly, Figure \ref{fig:subfig-a} is the initial graph, then in Figure \ref{fig:subfig-b} we consider an analogous structure of interest such that $\# \NF \leq 1$. Then, in Figure \ref{fig:subfig-c}, we consider an analogous structure of interest such that $\F \subseteq \{V_1\}$. Finally, in Figure \ref{fig:subfig-d} we  move every unused gray arrows such that only three vertices of $B^m$ are incident to any arrows. From this graph, we  deduce $\G^m_{\leq3}$ in Figure \ref{fig:subfig-e} on which d-connection with structures of interests are equivalent. We apply equivalent rules to get \ref{fig:subfig-f}, which corresponds to Figure \ref{fig:example}.

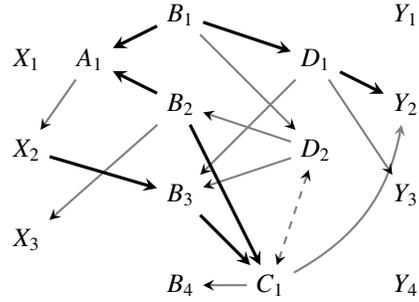
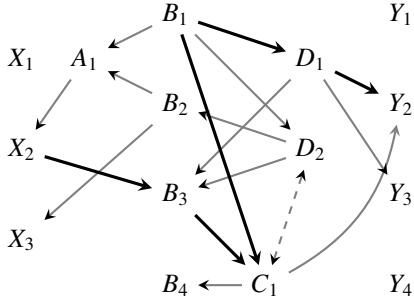
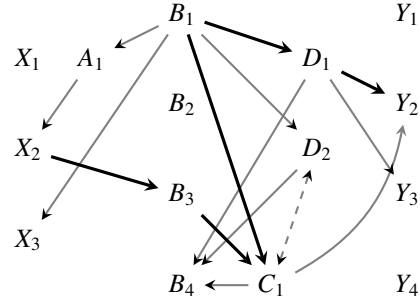
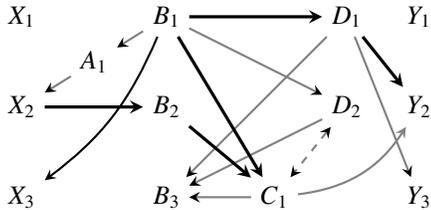
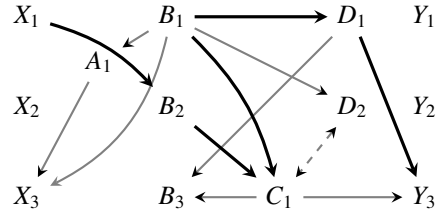
\begin{figure}[htbp]
  \centering
%
  \begin{subfigure}[t]{0.46\linewidth}
    \centering
    \begin{tikzpicture}[->, >=stealth, scale=1.2]
        \def\xX{0.3} \def\xA{1} \def\xB{2} \def\xC{3} \def\xD{3.5} \def\xY{4.5}

        \foreach \i in {1,...,3} {
            \node[text=black] (X\i) at (\xX, -0.5 -\i+1) {$X_{\i}$};
        }
        \node[text=black] (A1) at (\xA, -0.5) {$A_1$};
        \foreach \i in {1,...,4} {
            \node[text=black] (B\i) at (\xB, -\i+1) {$B_{\i}$};
        }
        \node[text=black] (C1) at (\xC, -3) {$C_1$};
        \foreach \i in {1,...,2} {
            \node[text=black] (D\i) at (\xD, -0.5 -\i+1) {$D_{\i}$};
        }
        \foreach \i in {1,...,4} {
            \node[text=black] (Y\i) at (\xY, -\i+1) {$Y_{\i}$};
        }

        \draw[->, draw=gray, line width=0.8pt] (C1) -- (B4);
        \draw[->, draw=gray, line width=0.8pt] (A1) -- (X2);
        \draw[->, draw=gray, line width=0.8pt, bend right=25] (C1) edge[gray] (Y2);
        \draw[<->, dashed, draw=gray, line width=0.8pt] (D2) -- (C1);
        \draw[->, draw=gray, line width=0.8pt] (B2) -- (X3);
        \draw[->, draw=gray, line width=0.8pt] (B1) -- (D2);
        \draw[->, draw=gray, line width=0.8pt] (D1) -- (B3);
        \draw[->, draw=gray, line width=0.8pt] (D1) -- (Y3);

        \draw[->, line width=1.2pt] (X2) -- (B3);
        \draw[->, line width=1.2pt] (D2) -- (B3);
        \draw[->, line width=1.2pt] (D2) -- (B2);
        \draw[->, line width=1.2pt] (B2) -- (A1);
        \draw[->, line width=1.2pt] (B1) -- (A1);
        \draw[->, line width=1.2pt] (B1) -- (D1);
        \draw[->, line width=1.2pt] (D1) -- (Y2);
        \draw[->, line width=1.2pt] (B3) -- (C1);
    \end{tikzpicture}
    \caption{$\Gm \in \C(\Gc)$}
    \label{fig:subfig-a}
  \end{subfigure}
\hfill
  \begin{subfigure}[t]{0.46\linewidth}
    \centering
    \begin{tikzpicture}[->, >=stealth, scale=1.2]
        \def\xX{0.3} \def\xA{1} \def\xB{2} \def\xC{3} \def\xD{3.5} \def\xY{4.5}

        \foreach \i in {1,...,3} {
            \node[text=black] (X\i) at (\xX, -0.5 -\i+1) {$X_{\i}$};
        }
        \node[text=black] (A1) at (\xA, -0.5) {$A_1$};
        \foreach \i in {1,...,4} {
            \node[text=black] (B\i) at (\xB, -\i+1) {$B_{\i}$};
        }
        \node[text=black] (C1) at (\xC, -3) {$C_1$};
        \foreach \i in {1,...,2} {
            \node[text=black] (D\i) at (\xD, -0.5 -\i+1) {$D_{\i}$};
        }
        \foreach \i in {1,...,4} {
            \node[text=black] (Y\i) at (\xY, -\i+1) {$Y_{\i}$};
        }

        \draw[->, draw=gray, line width=0.8pt] (C1) -- (B4);
        \draw[->, draw=gray, line width=0.8pt] (A1) -- (X2);
        \draw[->, draw=gray, line width=0.8pt, bend right=25] (C1) edge[gray] (Y2);
        \draw[<->, dashed, draw=gray, line width=0.8pt] (D2) -- (C1);
        \draw[->, draw=gray, line width=0.8pt] (B2) -- (X3);
        \draw[->, draw=gray, line width=0.8pt] (B1) -- (D2);
        \draw[->, draw=gray, line width=0.8pt] (D1) -- (B3);
        \draw[->, draw=gray, line width=0.8pt] (D1) -- (Y3);

        \draw[->, line width=1.2pt] (X2) -- (B3);
        \draw[->, draw=gray,line width=0.8pt] (D2) -- (B3);
        \draw[->, draw=gray,line width=0.8pt] (D2) -- (B2);
        \draw[->, line width=1.2pt] (B2) -- (A1);
        \draw[->, line width=1.2pt] (B1) -- (A1);
        \draw[->, line width=1.2pt] (B1) -- (D1);
        \draw[->, line width=1.2pt] (D1) -- (Y2);
        \draw[->, line width=1.2pt] (B3) -- (C1);
        \draw[->, line width=1.2pt] (B2) -- (C1);
    \end{tikzpicture}
    \caption{Application of Proposition~\ref{prop:move_arrow_up} to the arrow $B_3 \to C_1$, thereby adding $B_2 \to C_1$. Note that now $\#\NF = 1$.}
    \label{fig:subfig-b}
  \end{subfigure}

  \begin{subfigure}[t]{0.48\linewidth}
    \centering
    \begin{tikzpicture}[->, >=stealth, scale=1.2]
        \def\xX{0.3} \def\xA{1} \def\xB{2} \def\xC{3} \def\xD{3.5} \def\xY{4.5}

        \foreach \i in {1,...,3} {
            \node[text=black] (X\i) at (\xX, -0.5 -\i+1) {$X_{\i}$};
        }
        \node[text=black] (A1) at (\xA, -0.5) {$A_1$};
        \foreach \i in {1,...,4} {
            \node[text=black] (B\i) at (\xB, -\i+1) {$B_{\i}$};
        }
        \node[text=black] (C1) at (\xC, -3) {$C_1$};
        \foreach \i in {1,...,2} {
            \node[text=black] (D\i) at (\xD, -0.5 -\i+1) {$D_{\i}$};
        }
        \foreach \i in {1,...,4} {
            \node[text=black] (Y\i) at (\xY, -\i+1) {$Y_{\i}$};
        }

        \draw[->, draw=gray, line width=0.8pt] (C1) -- (B4);
        \draw[->, draw=gray, line width=0.8pt] (A1) -- (X2);
        \draw[->, draw=gray, line width=0.8pt, bend right=25] (C1) edge[gray] (Y2);
        \draw[<->, dashed, draw=gray, line width=0.8pt] (D2) -- (C1);
        \draw[->, draw=gray, line width=0.8pt] (B2) -- (X3);
        \draw[->, draw=gray, line width=0.8pt] (B1) -- (D2);
        \draw[->, draw=gray, line width=0.8pt] (D1) -- (B3);
        \draw[->, draw=gray, line width=0.8pt] (D1) -- (Y3);

        \draw[->, line width=1.2pt] (X2) -- (B3);
        \draw[->, draw=gray,line width=0.8pt] (D2) -- (B3);
        \draw[->, draw=gray,line width=0.8pt] (D2) -- (B2);
        \draw[->, draw=gray,line width=0.8pt] (B2) -- (A1);
        \draw[->, draw=gray,line width=0.8pt] (B1) -- (A1);
        \draw[->, line width=1.2pt] (B1) -- (D1);
        \draw[->, line width=1.2pt] (D1) -- (Y2);
        \draw[->, line width=1.2pt] (B3) -- (C1);
        \draw[->, line width=1.2pt] (B1) -- (C1);
    \end{tikzpicture}
    \caption{Application of Proposition~\ref{prop:move_arrow_up} to shift $B_2 \to C_1$ to $B_1 \to C_1$. Note that now $\F \subseteq \{V_1\}$.}
    \label{fig:subfig-c}
  \end{subfigure}
\hfill
%
  \begin{subfigure}[t]{0.48\linewidth}
    \centering
    \begin{tikzpicture}[->, >=stealth, scale=1.2]
        \def\xX{0.3} \def\xA{1} \def\xB{2} \def\xC{3} \def\xD{3.5} \def\xY{4.5}

        \foreach \i in {1,...,3} {
            \node[text=black] (X\i) at (\xX, -0.5 -\i+1) {$X_{\i}$};
        }
        \node[text=black] (A1) at (\xA, -0.5) {$A_1$};
        \foreach \i in {1,...,4} {
            \node[text=black] (B\i) at (\xB, -\i+1) {$B_{\i}$};
        }
        \node[text=black] (C1) at (\xC, -3) {$C_1$};
        \foreach \i in {1,...,2} {
            \node[text=black] (D\i) at (\xD, -0.5 -\i+1) {$D_{\i}$};
        }
        \foreach \i in {1,...,4} {
            \node[text=black] (Y\i) at (\xY, -\i+1) {$Y_{\i}$};
        }

        \draw[->, draw=gray, line width=0.8pt] (C1) -- (B4);
        \draw[->, draw=gray, line width=0.8pt] (A1) -- (X2);
        \draw[->, draw=gray, line width=0.8pt, bend right=25] (C1) edge[gray] (Y2);
        \draw[<->, dashed, draw=gray, line width=0.8pt] (D2) -- (C1);
        \draw[->, draw=gray, line width=0.8pt] (B1) -- (X3);
        \draw[->, draw=gray, line width=0.8pt] (B1) -- (D2);
        \draw[->, draw=gray, line width=0.8pt] (D1) -- (B4);
        \draw[->, draw=gray, line width=0.8pt] (D1) -- (Y3);

        \draw[->, line width=1.2pt] (X2) -- (B3);
        \draw[->, draw=gray,line width=0.8pt] (D2) -- (B4);
        \draw[->, draw=gray,line width=0.8pt] (B1) -- (A1);
        \draw[->, line width=1.2pt] (B1) -- (D1);
        \draw[->, line width=1.2pt] (D1) -- (Y2);
        \draw[->, line width=1.2pt] (B3) -- (C1);
        \draw[->, line width=1.2pt] (B1) -- (C1);
    \end{tikzpicture}
    \caption{Application of Propositions~\ref{prop:move_arrow_up} and~\ref{prop:move_arrow_down} to reposition all unused (gray) arrows. Note that only three vertices of $B^m$ are incident to any arrows.}
    \label{fig:subfig-d}
  \end{subfigure}
  
  \begin{subfigure}[t]{0.48\linewidth}
    \centering
    \begin{tikzpicture}[->, >=stealth, scale=1.2]
        \def\xX{0} \def\xA{0.8} \def\xB{1.6} \def\xC{2.8} \def\xD{3.6} \def\xY{4.4}

        \foreach \i in {1,...,3} {
            \node[text=black] (X\i) at (\xX, -\i+1) {$X_{\i}$};
        }
        \node[text=black] (A1) at (\xA, -.5) {$A_1$};
        \foreach \i in {1,...,3} {
            \node[text=black] (B\i) at (\xB, -\i+1) {$B_{\i}$};
        }
        \node[text=black] (C1) at (\xC, -2) {$C_1$};
        \foreach \i in {1,...,2} {
            \node[text=black] (D\i) at (\xD, -\i+1) {$D_{\i}$};
        }
        \foreach \i in {1,...,3} {
            \node[text=black] (Y\i) at (\xY, -\i+1) {$Y_{\i}$};
        }

        \draw[->, draw=gray, line width=0.8pt] (C1) -- (B3);
        \draw[->, draw=gray, line width=0.8pt] (A1) -- (X2);
        \draw[->, draw=gray, line width=0.8pt, bend right=25] (C1) edge[gray] (Y2);
        \draw[<->, dashed, draw=gray, line width=0.8pt] (D2) -- (C1);
        \path[->, draw=gray, line width=0.8pt] (B1) edge[bend left = 15] (X3);
        \draw[->, draw=gray, line width=0.8pt] (B1) -- (D2);
        \draw[->, draw=gray, line width=0.8pt] (D1) -- (B3);
        \draw[->, draw=gray, line width=0.8pt] (D1) -- (Y3);

        \path[->, line width=1.2pt] (X2) edge (B2);
        \draw[->, draw=gray,line width=0.8pt] (D2) -- (B3);
        \draw[->, draw=gray,line width=0.8pt] (B1) -- (A1);
        \draw[->, line width=1.2pt] (B1) -- (D1);
        \draw[->, line width=1.2pt] (D1) -- (Y2);
        \draw[->, line width=1.2pt] (B2) -- (C1);
        \draw[->, line width=1.2pt] (B1) -- (C1);
    \end{tikzpicture}
    \caption{Remove some vertices that are not incident to any arrow to get $\G^m_{\leq 3}$. Note that  $\G^m_{\leq 3} \in \C(\G^C_{\leq3})$.}
    \label{fig:subfig-e}
  \end{subfigure}
\hfill
%
  \begin{subfigure}[t]{0.48\linewidth}
    \centering
    \begin{tikzpicture}[->, >=stealth, scale=1.2]
        \def\xX{0} \def\xA{0.8} \def\xB{1.6} \def\xC{2.8} \def\xD{3.6} \def\xY{4.4}

        \foreach \i in {1,...,3} {
            \node[text=black] (X\i) at (\xX, -\i+1) {$X_{\i}$};
        }
        \node[text=black] (A1) at (\xA, -.5) {$A_1$};
        \foreach \i in {1,...,3} {
            \node[text=black] (B\i) at (\xB, -\i+1) {$B_{\i}$};
        }
        \node[text=black] (C1) at (\xC, -2) {$C_1$};
        \foreach \i in {1,...,2} {
            \node[text=black] (D\i) at (\xD, -\i+1) {$D_{\i}$};
        }
        \foreach \i in {1,...,3} {
            \node[text=black] (Y\i) at (\xY, -\i+1) {$Y_{\i}$};
        }

        \draw[->, draw=gray, line width=0.8pt] (C1) -- (B3);
        \draw[->, draw=gray, line width=0.8pt] (A1) -- (X3);
        \draw[->, draw=gray, line width=0.8pt] (C1) -- (Y3);
        \draw[<->,dashed, draw=gray, line width=0.8pt] (D2) -- (C1);
        \draw[->, draw=gray, line width=0.8pt, bend left=25] (B1) edge[gray] (X3);
        \draw[->, draw=gray, line width=0.8pt] (B1) -- (D2);
        \draw[->, draw=gray, line width=0.8pt] (D1) -- (B3);
        \draw[->, draw=gray, line width=0.8pt] (B1) -- (A1);

        \path[->, line width=1.2pt] (X1) edge[bend left=15] (B2);
        \draw[->, line width=1.2pt] (B1) -- (D1);
        \draw[->, line width=1.2pt] (D1) -- (Y3);
        \draw[->, line width=1.2pt] (B2) -- (C1);
        \path[->, line width=1.2pt] (B1) edge[bend left=15] (C1);

    \end{tikzpicture}
    \caption{Application of Propositions~\ref{prop:move_arrow_up} and~\ref{prop:move_arrow_down} to get the same graph of Figure~\ref{fig:example:2}.}
    \label{fig:subfig-f}
  \end{subfigure}

  \caption{Figure~\ref{fig:subfig-a} shows the graph $\Gm$ containing the structure of interest $\sigma^m$ (in bold black), which connects $X^m$ and $Y^m$ under $C^m \cup A^m$. Figures~\ref{fig:subfig-b}, \ref{fig:subfig-c} and \ref{fig:subfig-d} illustrate the successive transformations of $\Gm$ and $\sigma^m$ (as carried out in the proof of Theorem~\ref{th:infinity_leq_three}) for the cluster $B^C$. Figure~\ref{fig:subfig-e} shows the last step of the proof. Figure~\ref{fig:subfig-f} shows how to transform the graph in Figure~\ref{fig:subfig-e} to get the graph in Figure~\ref{fig:example:2}.}
  \label{fig:all-subfigs}
\end{figure}

\newpage
\section{NeurIPS Paper Checklist}

\begin{enumerate}

\item {\bf Claims}
    \item[] Question: Do the main claims made in the abstract and introduction accurately reflect the paper's contributions and scope?
    \item[] Answer: \answerYes{}{} 
    \item[] Justification: Main claims are listed in the introduction at line 62. They correspond to the outline of the paper. 

\item {\bf Limitations}
    \item[] Question: Does the paper discuss the limitations of the work performed by the authors?
    \item[] Answer: \answerYes{} 
    \item[] Justification: Limitations are discussed in the conclusion.

\item {\bf Theory assumptions and proofs}
    \item[] Question: For each theoretical result, does the paper provide the full set of assumptions and a complete (and correct) proof?
    \item[] Answer: \answerYes{} 
    \item[] Justification: All assumptions are described and all proofs are given in the Appendix.
    \item[] Guidelines:
    \begin{itemize}
        \item The answer NA means that the paper does not include theoretical results. 
        \item All the theorems, formulas, and proofs in the paper should be numbered and cross-referenced.
        \item All assumptions should be clearly stated or referenced in the statement of any theorems.
        \item The proofs can either appear in the main paper or the supplemental material, but if they appear in the supplemental material, the authors are encouraged to provide a short Sketch of proof to provide intuition. 
        \item Inversely, any informal proof provided in the core of the paper should be complemented by formal proofs provided in appendix or supplemental material.
        \item Theorems and Lemmas that the proof relies upon should be properly referenced. 
    \end{itemize}

    \item {\bf Experimental result reproducibility}
    \item[] Question: Does the paper fully disclose all the information needed to reproduce the main experimental results of the paper to the extent that it affects the main claims and/or conclusions of the paper (regardless of whether the code and data are provided or not)?
    \item[] Answer: \answerNA{} 
    \item[] Justification: The paper does not include experiments.

\item {\bf Open access to data and code}
    \item[] Question: Does the paper provide open access to the data and code, with sufficient instructions to faithfully reproduce the main experimental results, as described in supplemental material?
    \item[] Answer: \answerNA{} 
    \item[] Justification: The paper does not include experiments requiring code.

\item {\bf Experimental setting/details}
    \item[] Question: Does the paper specify all the training and test details (e.g., data splits, hyperparameters, how they were chosen, type of optimizer, etc.) necessary to understand the results?
    \item[] Answer: \answerNA{} 
    \item[] Justification: The paper does not include experiments.

\item {\bf Experiment statistical significance}
    \item[] Question: Does the paper report error bars suitably and correctly defined or other appropriate information about the statistical significance of the experiments?
    \item[] Answer: \answerNA{} 
    \item[] Justification: The paper does not include experiments.

\item {\bf Experiments compute resources}
    \item[] Question: For each experiment, does the paper provide sufficient information on the computer resources (type of compute workers, memory, time of execution) needed to reproduce the experiments?
    \item[] Answer: \answerNA{} 
    \item[] Justification: The paper does not include experiments.
    
\item {\bf Code of ethics}
    \item[] Question: Does the research conducted in the paper conform, in every respect, with the NeurIPS Code of Ethics \url{https://neurips.cc/public/EthicsGuidelines}?
    \item[] Answer: \answerYes{} 
    \item[] Justification: The paper conform, in every respect, with the NeurIPS Code of Ethics

\item {\bf Broader impacts}
    \item[] Question: Does the paper discuss both potential positive societal impacts and negative societal impacts of the work performed?
    \item[] Answer: \answerNA{} 
    \item[] Justification: The paper is a theoretical study of completeness of calculus. 
    
\item {\bf Safeguards}
    \item[] Question: Does the paper describe safeguards that have been put in place for responsible release of data or models that have a high risk for misuse (e.g., pretrained language models, image generators, or scraped datasets)?
    \item[] Answer: \answerNA{} 
    \item[] Justification: The paper poses no such risks.

\item {\bf Licenses for existing assets}
    \item[] Question: Are the creators or original owners of assets (e.g., code, data, models), used in the paper, properly credited and are the license and terms of use explicitly mentioned and properly respected?
    \item[] Answer: \answerNA{} 
    \item[] Justification: The paper does not use existing assets.

\item {\bf New assets}
    \item[] Question: Are new assets introduced in the paper well documented and is the documentation provided alongside the assets?
    \item[] Answer: \answerNA{} 
    \item[] Justification: The paper does not release new assets.

\item {\bf Crowdsourcing and research with human subjects}
    \item[] Question: For crowdsourcing experiments and research with human subjects, does the paper include the full text of instructions given to participants and screenshots, if applicable, as well as details about compensation (if any)? 
    \item[] Answer: \answerNA{} 
    \item[] Justification: The paper does not involve crowdsourcing nor research with human subjects.

\item {\bf Institutional review board (IRB) approvals or equivalent for research with human subjects}
    \item[] Question: Does the paper describe potential risks incurred by study participants, whether such risks were disclosed to the subjects, and whether Institutional Review Board (IRB) approvals (or an equivalent approval/review based on the requirements of your country or institution) were obtained?
    \item[] Answer: \answerNA{} 
    \item[] Justification: The paper does not involve crowdsourcing nor research with human subjects.

\item {\bf Declaration of LLM usage}
    \item[] Question: Does the paper describe the usage of LLMs if it is an important, original, or non-standard component of the core methods in this research? Note that if the LLM is used only for writing, editing, or formatting purposes and does not impact the core methodology, scientific rigorousness, or originality of the research, declaration is not required.
    \item[] Answer: \answerNA{} 
    \item[] Justification: The core method development in this research does not involve LLMs as any important, original, or non-standard components.

\end{enumerate}

\end{document}